\documentclass[11pt]{article}

\usepackage[preprint]{acl}

\usepackage{times}
\usepackage{latexsym}

\usepackage[T1]{fontenc}

\usepackage[utf8]{inputenc}

\usepackage{microtype}

\usepackage{inconsolata}

\usepackage{graphicx}
\usepackage{enumitem}
\usepackage{tabularx}
\usepackage{multirow} 
\usepackage{adjustbox}
\usepackage{booktabs}
\usepackage{amsthm}
\usepackage{amsmath}
\usepackage[table]{xcolor}
\usepackage{colortbl}
\usepackage{amssymb}

\usepackage{pgfplots}
\pgfplotsset{compat=1.18}
\usepgfplotslibrary{groupplots}
\usepgfplotslibrary{statistics}
\usetikzlibrary{patterns}

\theoremstyle{definition}

\theoremstyle{plain}

\newtheorem{theorem}{Theorem}

\usepackage[most,skins,theorems]{tcolorbox}

\usepackage{listings}
\usepackage{xcolor}
\usepackage[most]{tcolorbox}
\usepackage{xcolor}

\definecolor{userbg}{RGB}{245,247,250}
\definecolor{assistantbg}{RGB}{250,250,250}
\definecolor{reasonbg}{RGB}{243,248,255}

\newtcolorbox{chatuser}{
  colback=userbg,
  colframe=black!20,
  boxrule=0.4pt,
  arc=2pt,
  left=6pt,right=6pt,top=4pt,bottom=4pt,
  fontupper=\small
}

\newtcolorbox{chatassistant}{
  colback=assistantbg,
  colframe=black!20,
  boxrule=0.4pt,
  arc=2pt,
  left=6pt,right=6pt,top=4pt,bottom=4pt,
  fontupper=\small
}

\newtcolorbox{chatreason}{
  colback=reasonbg,
  colframe=black!20,
  boxrule=0.4pt,
  arc=2pt,
  left=6pt,right=6pt,top=4pt,bottom=4pt,
  fontupper=\small
}

\usepackage{tikz}
\usepackage{pgfplots}
\pgfplotsset{compat=1.18}
\usepgfplotslibrary{polar}

\newcolumntype{L}[1]{>{\raggedright\arraybackslash}p{#1}}
\newcommand{\grouphead}[1]{%
\addlinespace[0.35em]
\multicolumn{17}{c}{\textit{#1}}\\
\addlinespace[0.15em]\midrule
}

\title{EmoLLM: Appraisal-Grounded Cognitive-Emotional Co-Reasoning in Large Language Models}

\author{
Yifei Zhang, Mingyang Li, Henry Gao, Liang Zhao \\
Department of Computer Science, Emory University \\
\texttt{\{yifei.zhang2, mingyang.li, henry.gao2, liang.zhao\}@emory.edu}
}

\begin{document}
\maketitle

\begin{abstract}
Large language models (LLMs) demonstrate strong cognitive intelligence (IQ), yet many real-world interactions also require emotional intelligence (EQ) to produce responses that are both factually reliable and emotionally appropriate. In settings such as emotional support, technical assistance, and consultation, effective dialogue depends on how situations are appraised with respect to the user’s needs, goals, and coping capacity. Inspired by appraisal theory, we propose EmoLLM, an appraisal-grounded framework for IQ--EQ co-reasoning in dialogue. EmoLLM uses an explicit Appraisal Reasoning Graph (ARG) to structure intermediate reasoning over contextual facts, inferred user needs, appraisal dimensions, emotional states, and response strategies before generating a reply. We train EmoLLM in a multi-turn role-play environment with reinforcement learning, where reverse-perspective reasoning provides reward signals based on predicted user-side consequences of responses. Across diverse dialogue settings, EmoLLM improves emotional state outcomes and response quality over strong baselines while preserving strong factual reliability.
Code available at \url{https://github.com/yifeizhangcs/EmoLLM}.
\end{abstract}

\section{Introduction}
\label{sec:intro}

\begin{figure}[t]
    \centering
    \centerline{\includegraphics[width=\linewidth]{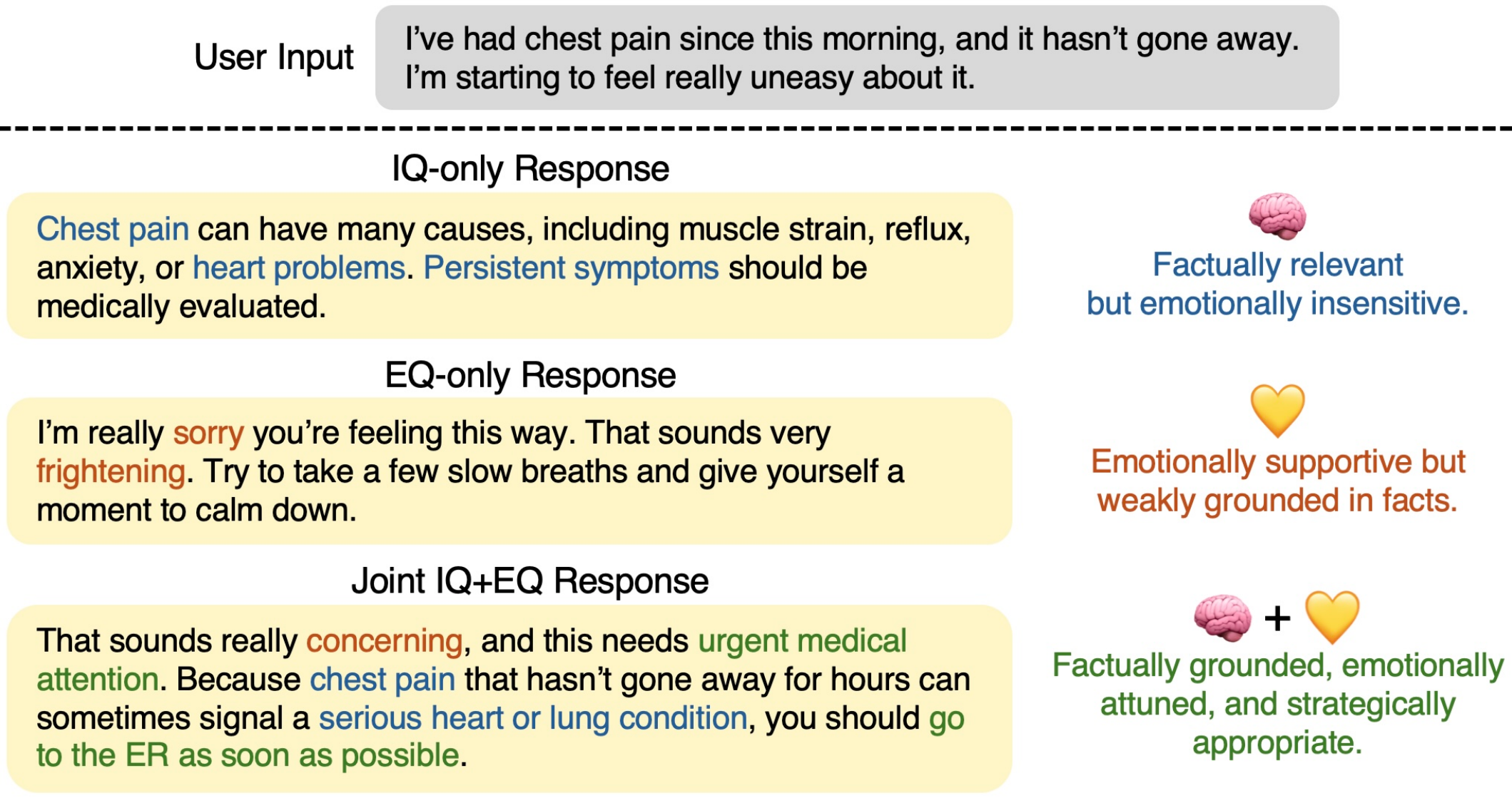}}
    \caption{\textbf{Why IQ--EQ co-reasoning matters.}
    IQ-only responses can be factually relevant but emotionally insensitive, while EQ-only responses can be emotionally supportive but insufficiently grounded in the underlying situation. 
    IQ--EQ co-reasoning enables responses that are factually grounded, emotionally attuned, and strategically appropriate.}
    \label{fig:joint_reasoning}
\end{figure}

Large language models (LLMs)~\cite{achiam2023gpt,grattafiori2024llama,yang2025qwen3,team2025kimi} have demonstrated strong \emph{cognitive intelligence} (IQ), with remarkable progress in reasoning and problem solving across mathematics, programming, and logical inference benchmarks~\cite{wei2022chain,ahn2024large,he2024olympiadbench}. At the same time, a growing body of work has improved LLMs' \emph{emotional intelligence} (EQ), including emotion recognition, empathetic response generation, and affect-aware dialogue~\cite{wang2023emotional,li2023large,sorin2024large}. However, many real-world interactions require cognitive and emotional capabilities to operate jointly rather than in isolation. In scenarios such as healthcare, counseling, and education, appropriate responses must \emph{reason jointly} about the user’s emotional state and the situational context that gives rise to it, rather than treating emotion as a surface attribute independent of context. As illustrated in Figure~\ref{fig:joint_reasoning}, integrating these dimensions enables models to interpret situational factors together with emotional signals, yielding responses that are both factually grounded and emotionally appropriate, instead of being driven by either cognitive reasoning or emotional cues alone.

Existing work has improved both cognitive reasoning and emotional intelligence in LLMs, but largely treats them separately. Advances in scaling and reasoning-oriented training have improved analytical problem solving~\cite{wei2022chain,ahn2024large,he2024olympiadbench}, while prior work has enhanced emotional intelligence through supervised fine-tuning, emotion-aware prompting, knowledge augmentation, and alignment with social or user-centered objectives~\cite{rashkin2019towards,zheng2023building,qian2023harnessing,rakib2025dialogxpert}. However, these approaches rarely model how situational facts, inferred user needs, and emotional understanding jointly shape response strategies. As a result, the structured process linking contextual interpretation, appraisal, emotion, and strategy remains underexplored.

Psychological research offers a useful perspective for addressing this gap. \emph{Appraisal Theory} suggests that emotions arise not directly from events themselves, but from how those events are evaluated with respect to an individual's goals, needs, and coping potential~\cite{lazarus1991emotion,scherer2001appraisal,ortony2022cognitive}. Related work on perspective-taking further suggests that emotionally appropriate responses require reasoning about the user's internal state and situational interpretation~\cite{kosinski2023theory}. These perspectives imply that appropriate responses emerge from a structured process linking contextual facts, inferred goals and needs, evaluative appraisals, emotional states, and response strategies. Applying this perspective to LLMs presents three key challenges: representing appraisal as an explicit intermediate reasoning process, jointly improving cognitive reliability and emotional appropriateness, and training multi-turn behavior using user-side consequence signals rather than final-response quality alone.

To address these challenges, we propose \emph{EmoLLM}, an appraisal-grounded framework for cognitive--emotional (IQ--EQ) co-reasoning in dialogue. At its core, we introduce the \emph{Appraisal Reasoning Graph} (ARG), an explicit structure organizing contextual facts, inferred user needs, appraisals, emotions, and response strategies. We further introduce reverse-perspective reasoning, which predicts how a response updates the user's needs, appraisals, and emotional state, providing reward signals based on the plausibility of this predicted transition. Building on this, we train EmoLLM in a multi-turn role-play environment with reinforcement learning using reverse-perspective and ARG-derived structure-aware rewards. Experiments across diverse dialogue benchmarks show that EmoLLM consistently outperforms strong baselines on emotional state outcomes, empathic appropriateness, and factual reliability.

\section{Problem Formulation}
\label{sec:problem}

We study \emph{IQ--EQ co-reasoning} in multi-turn dialogue. At turn $t$, let $x_{\le t} = (u_1, y_1, \ldots, u_{t-1}, y_{t-1}, u_t)$ denote the dialogue history up to the current user utterance, where $u_i$ and $y_i$ are the user and assistant utterances at turn $i$, respectively. Given $x_{\le t}$, the model generates a response $y_t$ that is both factually reliable and emotionally appropriate. We model IQ and EQ as coupled latent factors rather than separable objectives. Let $z_t^{\mathrm{IQ}}$ and $z_t^{\mathrm{EQ}}$ denote latent cognitive and emotional reasoning states, and define:
\begin{align}
p_\theta(y_t \mid x_{\le t})
&= \sum_{z_t^{\mathrm{IQ}},\, z_t^{\mathrm{EQ}}}
p_\theta(z_t^{\mathrm{IQ}}, z_t^{\mathrm{EQ}} \mid x_{\le t}) \notag \\
&\quad \cdot
p_\theta(y_t \mid x_{\le t}, z_t^{\mathrm{IQ}}, z_t^{\mathrm{EQ}})
\label{eq:problem_factorization}
\end{align}
where the joint prior $p_\theta(z_t^{\mathrm{IQ}}, z_t^{\mathrm{EQ}} \mid x_{\le t})$ captures IQ--EQ coupling, i.e., we do not assume that it factorizes into independent IQ and EQ terms. Although formulated at the turn level, the task requires jointly modeling factual reliability and emotional appropriateness under the evolving dialogue context $x_{\le t}$. Our objective is to learn $\theta$ such that responses remain factually reliable while also being emotionally grounded through coherent IQ--EQ interactions. A key challenge is that standard response-only training signals (e.g., likelihood or scalar turn-level scores) provide limited guidance for learning structured IQ--EQ coupling and for evaluating the longer-horizon user-side consequences of responses.


\section{Methodology}
\label{sec:methods}

To model dialogue as an appraisal-grounded process of cognitive--emotional co-reasoning, we introduce the \emph{Appraisal Reasoning Graph} (ARG) (\S\ref{sec:framework}), an explicit dependency structure that organizes intermediate reasoning from situational understanding to response strategy selection. This graph forms the core reasoning structure of \emph{EmoLLM} and guides response generation before the final reply is produced. Training of EmoLLM follows a two-stage pipeline (\S\ref{sec:train_pipeline}) that first initializes this structured co-reasoning process and then refines it through multi-turn role-play with \emph{reverse-perspective} reasoning and reinforcement learning.

\begin{figure*}[!ht] 
\centerline{\includegraphics[width=\textwidth]{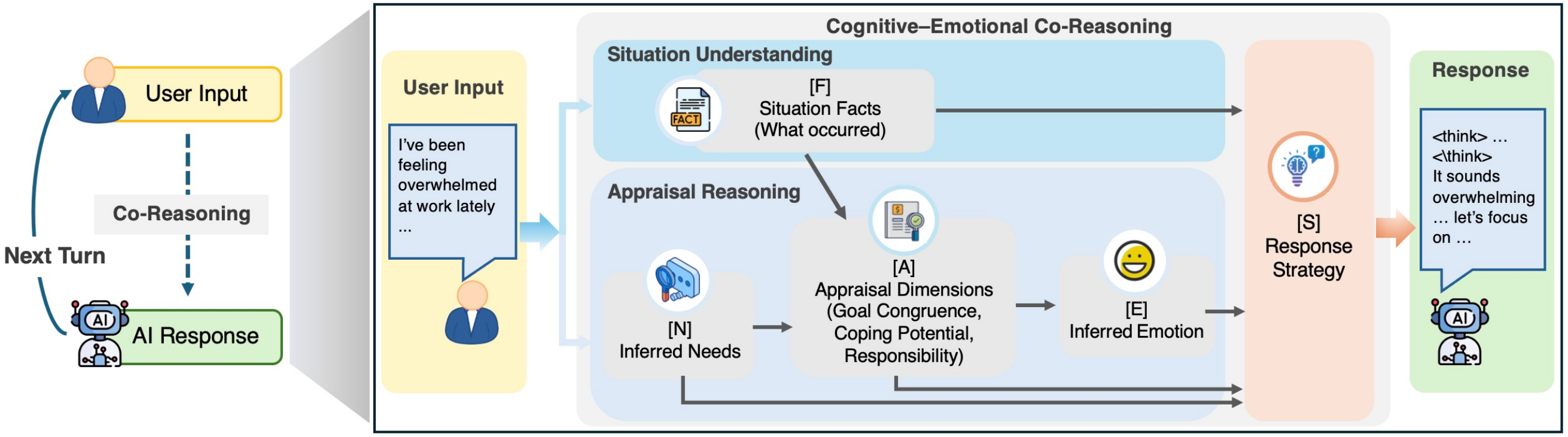}}
\caption{
\textbf{Appraisal Reasoning Graph (ARG) in EmoLLM.}
At each dialogue turn, EmoLLM instantiates an ARG from the dialogue context to perform appraisal-grounded cognitive--emotional co-reasoning before generating a reply.
The process repeats across turns in multi-turn interaction.
}
\label{fig:emollm_framework} 
\end{figure*}

\subsection{Appraisal Reasoning Graph}
\label{sec:framework}

At each dialogue turn, EmoLLM instantiates an \emph{Appraisal Reasoning Graph} (ARG) from the dialogue context to guide response planning and generation.

\paragraph{Motivation from Appraisal Theory.}
Given dialogue context $x_{\le t}$, appraisal-grounded co-reasoning decomposes response generation into five interdependent components: contextual facts $F_t$, inferred user needs or goals $N_t$, appraisal dimensions $A_t$, emotional states $\mathcal{E}_t$, and response strategies $S_t$. This decomposition is motivated by appraisal theory, which views emotion as arising not directly from events themselves, but from evaluations of situational conditions with respect to an individual’s goals, needs, and coping potential~\cite{lazarus1991emotion,scherer2001appraisal,ortony2022cognitive}. This perspective motivates the graph structure defined below.

\paragraph{ARG Definition.}
The Appraisal Reasoning Graph (ARG), illustrated in Figure~\ref{fig:emollm_framework}, is defined as a directed acyclic graph $G^{\mathrm{ARG}}=(V,E_{\mathrm{ARG}})$ with node set $V=\{F,N,A,\mathcal{E},S\}$, representing contextual facts, inferred user needs or goals, appraisal dimensions, emotional state, and response strategy. The dependency structure $E_{\mathrm{ARG}}$ is fixed across dialogue turns and is specified as
\[
(F,N) \rightarrow A,\; A \rightarrow \mathcal{E},\; (F,N,A,\mathcal{E}) \rightarrow S.
\]
At turn $t$, the graph is instantiated from dialogue history $x_{\le t}$ as $(F_t,N_t,A_t,\mathcal{E}_t,S_t)$. Under this instantiation, the abstract latent states in Eq.~\eqref{eq:problem_factorization} are grounded through ARG nodes: the IQ-side state $z_t^{\mathrm{IQ}}$ is instantiated primarily through contextual facts $F_t$, whereas the EQ-side state $z_t^{\mathrm{EQ}}$ is instantiated through the appraisal-grounded user state $(N_t,A_t,\mathcal{E}_t)$. The resulting structured state is
\[
z_t^{\mathrm{ARG}}=(F_t,N_t,A_t,\mathcal{E}_t,S_t).
\]
As the dialogue proceeds and $x_{\le t}$ grows, node values evolve under this shared graph template, enabling cognitive--emotional co-reasoning across multiple turns.

\paragraph{Structured Factorization.}

Eq.~\eqref{eq:problem_factorization} formulates response generation through two coupled latent variables, $z_t^{\mathrm{IQ}}$ and $z_t^{\mathrm{EQ}}$. In EmoLLM, the ARG provides an explicit structure for these variables: $F_t$ instantiates the cognitive state $z_t^{\mathrm{IQ}}$, while $(N_t, A_t, \mathcal{E}_t)$ instantiates the appraisal-grounded emotional state $z_t^{\mathrm{EQ}}$. Crucially, $S_t$ acts as the \emph{binding variable} that bridges joint IQ--EQ latent space, translating the internal cognitive--emotional understanding into an actionable response. We therefore write
$$
\begin{aligned}
&p_\theta(y_t \mid x_{\le t}) = \sum\nolimits_{F_t,N_t,A_t,\mathcal{E}_t,S_t} \\
& p_\theta(F_t \mid x_{\le t})\, p_\theta(N_t \mid x_{\le t})\, p_\theta(A_t \mid F_t,N_t) \\
&\, \cdot p_\theta(\mathcal{E}_t \mid A_t)\, p_\theta(S_t \mid F_t,N_t,A_t,\mathcal{E}_t) \\
&\, \cdot p_\theta(y_t \mid x_{\le t},F_t,N_t,A_t,\mathcal{E}_t,S_t).
\end{aligned}
$$
This factorization structurally instantiates Eq.~\eqref{eq:problem_factorization}. IQ--EQ coupling is preserved because appraisal depends on both facts and needs, and the strategy $S_t$ translates these joint states into the response $y_t$.

\subsection{Training Framework}
\label{sec:train_pipeline}

Training of EmoLLM follows a two-stage pipeline that progressively establishes and refines cognitive--emotional co-reasoning. Stage~I (\S\ref{sec:stage_one}) initializes appraisal-grounded reasoning through knowledge-guided pretraining and teacher-guided supervision. Stage~II (\S\ref{sec:rl_roleplay}) further improves interaction-level behavior through reinforcement learning in a multi-turn role-play environment, where reverse-perspective reasoning evaluates the predicted user-side consequences of candidate responses.

\subsubsection{Stage I: Appraisal-Grounded Co-Reasoning Initialization}
\label{sec:stage_one}

Stage~I establishes appraisal-grounded co-reasoning prior to reinforcement learning. 
The stage consists of two components: (i) knowledge-grounded continued pretraining that injects scalable appraisal signals, and (ii) ARG-guided supervised initialization that anchors high-quality appraisal-to-strategy reasoning.

\paragraph{Knowledge-Grounded Continued Pretraining.}
Continued pretraining (CPT) is performed on dialogue corpora augmented with structured affective knowledge. For each dialogue context $x_{\le t}$, a knowledge-grounded tuple is constructed as
\[
g_t=\mathcal{G}(x_{\le t},\mathcal{K})=(F_t, N_t, A_t, \mathcal{E}_t),
\]
where $\mathcal{G}(\cdot)$ extracts salient situational facts and retrieves relevant knowledge from $\mathcal{K}$ in order to infer user needs or goals, appraisal dimensions, and the corresponding emotional state. A linearized form of $g_t$ is then prepended as a short structured prefix to the original context,
\[
\tilde{x}_t=\texttt{<think>}~g_t~\texttt{</think>} \,\Vert\, x_{\le t},
\]
and training continues with a standard causal language modeling objective. This procedure provides scalable weak supervision over appraisal states without requiring manual annotation.

\paragraph{ARG-Guided Supervised Initialization.}
To complement large-scale weak supervision, a small subset of contexts is sampled and high-quality supervision is obtained from a strong teacher LLM (for example, a proprietary API model). For each sampled context $x_{\le t}$, the teacher generates an ARG state $z_t^{\mathrm{ARG}}$ together with a gating label $r_t\in\{0,1\}$ indicating whether explicit structured reasoning should be emitted for that turn. The supervision target is defined as
\[
\tilde{y}_t =
\begin{cases}
\texttt{<think>}~z_t^{\mathrm{ARG}}~\texttt{</think>}~y_t, & \text{if } r_t=1, \\
y_t, & \text{if } r_t=0,
\end{cases}
\]
and the student model is fine-tuned to reproduce $\tilde{y}_t$ given $x_{\le t}$. Compared with the knowledge-grounded tuple $g_t$, the supervised ARG state additionally includes the response strategy $S_t$, which enables direct supervision of the mapping from appraisal states to response strategies.

\begin{figure*}[ht]
\centerline{\includegraphics[width=0.9\textwidth]{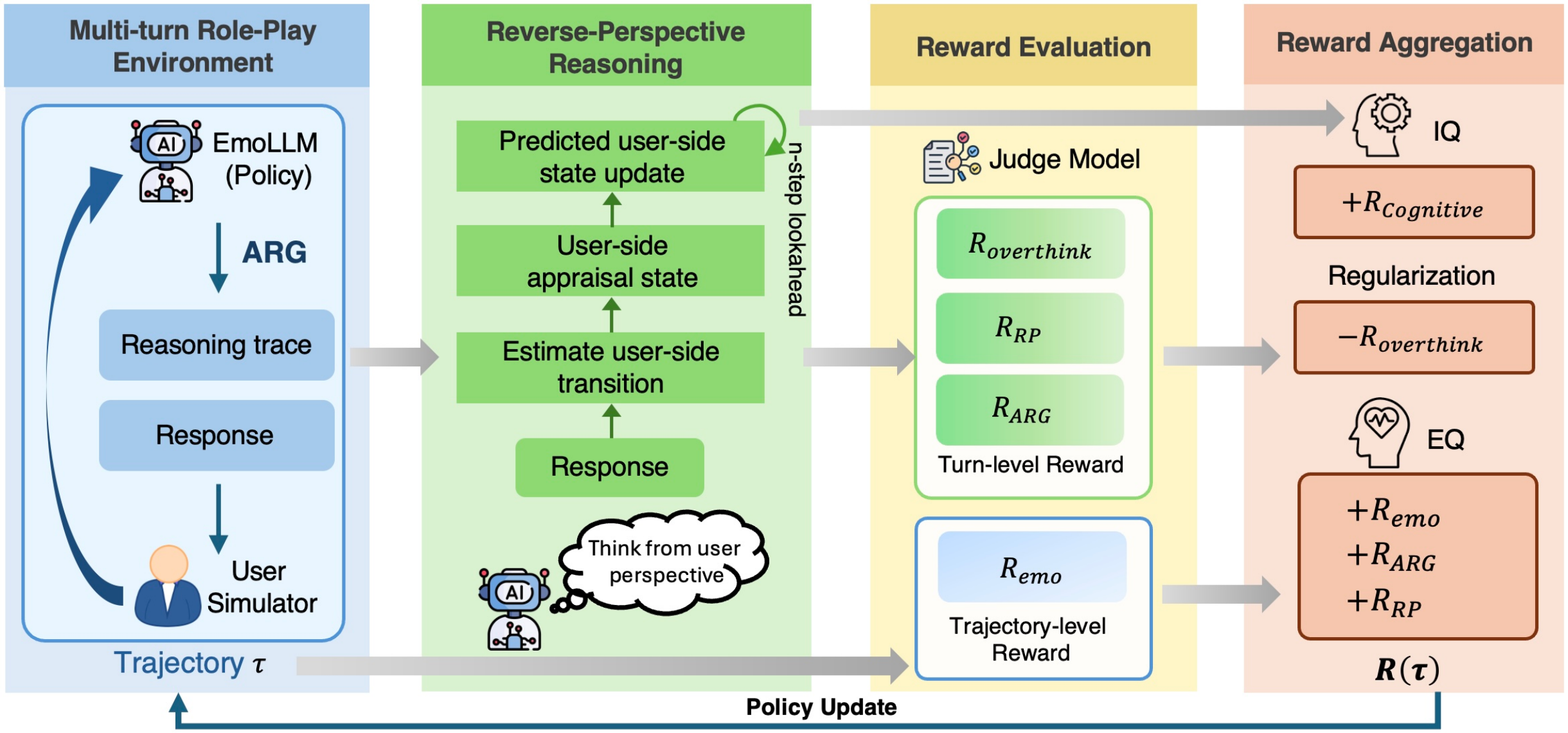}}
\caption{\textbf{Stage II: Multi-turn RL with reverse-perspective reasoning.} The policy interacts with a user simulator to generate dialogue trajectories. For each response, the model performs reverse-perspective reasoning to estimate the induced user-side transition in needs, appraisals, and emotions, optionally with $n$-step lookahead. A judge model evaluates the predicted transition to produce reverse-perspective reward signals for policy optimization.}
\label{fig:rl}
\end{figure*}

\subsubsection{Stage II: Multi-Turn RL with Reverse-Perspective Reasoning}
\label{sec:rl_roleplay}

Stage~II builds on Stage~I's appraisal-grounded initialization and further refines EmoLLM through reinforcement learning in a multi-turn role-play environment (Figure~\ref{fig:rl}). Whereas Stage~I enforces turn-level ARG-consistent appraisal-to-strategy reasoning, Stage~II improves behavior over entire interactions by training the policy with reward signals derived from the predicted downstream effects on the user. After generating a response, the policy performs internal reverse-perspective reasoning to estimate how the user's needs, appraisals, and emotional state may evolve.

\paragraph{Multi-Turn Role-Play Environment.}
The policy interacts with a user simulator to generate a dialogue trajectory
$\tau=\{(x_t, z_t^{\mathrm{ARG}}, y_t)\}_{t=1}^T,$
where $x_t$ denotes the turn context, $z_t^{\mathrm{ARG}}$ denotes the ARG reasoning trace produced during training, and $y_t$ denotes the assistant response. Given $(x_t, y_t)$, the simulator generates the next user utterance and updates the dialogue context, yielding multi-turn trajectories for reward computation.

\paragraph{Reverse-Perspective Reasoning.}
After generating response $y_t$ at turn $t$, an \emph{internal} reverse-perspective reasoning step estimates its user-side consequence. Let $x_{\le t}$ denote the dialogue history up to turn $t$. Using the underlying model under a specific user-role prompt, we predict a compact summary of the induced user-side transition:
\begin{equation*}
\begin{aligned}
& \hat{s}_t^{(1)} = \mathcal{T}_\theta(x_{\le t}, y_t), \\
& \hat{s}_t^{(k+1)} = \mathcal{T}_\theta(x_{\le t}, y_t, \hat{s}_t^{(k)}), \; k=1,\dots,n-1,
\end{aligned}
\end{equation*}
where $\mathcal{T}_\theta$ is the transition predictor and $\hat{s}_t^{(k)}$ denotes the simulated user state at step $k$. After $n$ steps, this process yields $\hat{s}_t^{(n)}=(\hat{N}_t,\hat{A}_t,\hat{\mathcal{E}}_t),$
which summarizes the predicted updates in the user’s needs, appraisals, and emotional state. A judge model then scores this transition to define the reverse-perspective reward $R_{\mathrm{RP}}(t)=\mathcal{S}\!\left(\hat{s}_t^{(n)}\right),$
where $\mathcal{S}$ evaluates plausibility and contextual consistency. Reverse-perspective reasoning is used only during training to estimate the user-side consequence and provide reward signals; at inference time, the policy generates responses directly. Additional implementation details are provided in Appendix~\ref{appendix:details_rp}.

When $n=1$, this process evaluates the immediate predicted user reaction, whereas larger $n$ enables deeper lookahead over downstream consequences. We can abstract this multi-turn evolution as a latent-state discounted MDP, where state $s_t$ encapsulates the dialogue history and latent user variables, and response $y_t$ serves as action $a_t$. Under this abstraction, the following theorem formalizes how increasing lookahead depth $n$ theoretically bounds the myopic evaluation bias.

\begin{theorem}[Under the latent-state MDP abstraction, lookahead depth reduces truncation bias]
\label{thm:rollout_bias}
Assume a discounted MDP with bounded rewards $|r(s,a)|\le R_{\max}$ and discount factor $\gamma\in(0,1)$. Let $Q_{\pi}(s,a)$ denote the true action value under policy $\pi$, and let $Q_{\pi}^{(n)}(s,a)$ denote its $n$-step truncated return. Then, for any policy $\pi$ and any state--action pair $(s,a)$,
\begin{equation*}
\bigl|Q_{\pi}(s,a)-Q_{\pi}^{(n)}(s,a)\bigr|
\le
\frac{\gamma^n}{1-\gamma}R_{\max}.
\end{equation*}
\end{theorem}
\noindent\textit{Proof.} See Appendix~\ref{appendix:theory_rp_lookahead}. $\square$

This result suggests that deeper reverse-perspective reasoning reduces myopic evaluation bias, although larger $n$ may accumulate compounding simulation errors from the predictor $\mathcal{T}_\theta$. Shallow lookahead is therefore adopted in practice to balance the lookahead horizon against simulation variance.

\paragraph{Reward Evaluation.}
At each turn, a judge model evaluates four aspects: cognitive reliability, ARG trace quality, reverse-perspective plausibility, and an \emph{overthinking penalty}. This penalty discourages redundant computation, such as invoking an ARG trace for simple queries or performing excessive dynamic $n$-step lookahead. Finally, a trajectory-level reward evaluates the predicted user emotional trajectory (Appendix~\ref{appendix:reward_rubrics}).

\paragraph{Reward Aggregation and Optimization.}
These rewards are aggregated over the trajectory, and a weighted objective is optimized using trajectory-level GRPO; detailed aggregation and optimization formulas are provided in Appendix~\ref{appendix:grpo_objective}. This objective encourages responses that are cognitively reliable, appraisal--strategy coherent, and emotionally appropriate under multi-turn interaction.

\begin{table*}[ht]
\caption{Performance comparison across four datasets using Qwen3--8B.}
\centering
\small
\setlength{\tabcolsep}{1.6pt}
\renewcommand{\arraystretch}{1.02} 
\newcommand{\methodcolw}{4.4cm}     
\resizebox{\textwidth}{!}{%
\begin{tabular}{L{\methodcolw}|cccc|cccc|cccc|cccc}
\toprule
\textbf{Method}
& \multicolumn{4}{c|}{\textbf{ED}}
& \multicolumn{4}{c|}{\textbf{MSD}}
& \multicolumn{4}{c|}{\textbf{MedD}}
& \multicolumn{4}{c}{\textbf{ICLR}} \\
& SR$\uparrow$ & AT$\downarrow$ & ES$\uparrow$ & EA$\uparrow$
& SR$\uparrow$ & AT$\downarrow$ & ES$\uparrow$ & EA$\uparrow$
& SR$\uparrow$ & AT$\downarrow$ & ES$\uparrow$ & EA$\uparrow$
& SR$\uparrow$ & AT$\downarrow$ & ES$\uparrow$ & EA$\uparrow$ \\
\midrule
\midrule
\grouphead{Base Model}
\rowcolor{gray!15}
Qwen3-8B         
& 88.7 & 3.03 & 4.58 & 3.84 
& 77.6 & 3.27 & 4.13 & 3.41 
& 89.2 & 2.92 & 3.95 & 3.57 
& 83.4 & 1.64 & 3.79 & 3.83 \\
\midrule
\grouphead{Prompting Methods}
Best-of-$N$
& 90.4 & 2.78 & 4.48 & 3.72
& 79.6 & 3.12 & 4.08 & 3.51
& 90.8 & 2.41 & 3.92 & 3.97
& 84.9 & 1.71 & 3.77 & 3.64 \\
ARG-Prompt
& 91.0 & 2.41 & 4.67 & 4.10
& 81.0 & 2.88 & 4.14 & 3.58
& 92.3 & 2.28 & 4.02 & 4.01
& 86.0 & 1.39 & 3.95 & 3.82 \\
ICL-AIF~\cite{fu2023improving}
& 91.3 & 2.29 & 4.73 & 4.24
& 81.5 & 2.81 & 4.16 & 3.62
& 92.7 & 2.31 & 4.06 & 4.06
& 86.4 & 1.35 & 3.99 & 3.86 \\
Ask-an-Expert~\cite{zhang2023ask}
& 89.4 & 2.71 & 4.51 & 3.97
& 79.1 & 3.15 & 4.05 & 3.46
& 90.8 & 2.63 & 3.96 & 3.72
& 84.2 & 1.58 & 3.85 & 3.66 \\
\grouphead{SFT Methods}
SFT              
& 87.3 & 2.14 & 4.39 & 4.07 
& 81.2 & 2.74 & 3.97 & 3.37 
& 90.4 & 2.29 & 3.91 & 3.68 
& 82.6 & 1.36 & 3.73 & 3.76 \\
ESCoT~\cite{rashkin2019towards}
& 89.7 & 2.64 & 4.66 & 3.88
& 79.4 & 3.06 & 4.18 & 3.47
& 91.2 & 2.58 & 4.01 & 3.76
& 84.1 & 1.47 & 3.86 & 3.73 \\
PPDP~\cite{deng2023plug}
& 90.2 & 2.41 & 4.63 & 4.12
& 80.4 & 2.93 & 4.11 & 3.54
& 91.8 & 2.37 & 4.00 & 3.89
& 85.3 & 1.42 & 3.90 & 3.79 \\
\grouphead{RL Methods}
RLFF-ESC~\cite{yang2025towards}
& 90.8 & 2.49 & 4.71 & 4.02
& 80.6 & 2.96 & 4.15 & 3.55
& 92.4 & 2.44 & 4.02 & 3.92
& 86.0 & 1.39 & 3.95 & 3.78 \\
RLVER~\cite{wang2025rlverreinforcementlearningverifiable}
& 91.1 & 2.38 & 4.74 & 4.10
& 81.2 & 2.88 & 4.18 & 3.61
& 93.0 & 2.36 & 4.08 & 4.01
& 86.6 & 1.33 & 4.02 & 3.81 \\
PERM~\cite{wang2026perm}
& 91.8 & 2.22 & 4.79 & 4.28
& 82.1 & 2.74 & 4.20 & 3.64
& 94.0 & 2.28 & 4.10 & 4.18
& 88.3 & 1.26 & 4.05 & 3.88 \\
\grouphead{Ours}
\textbf{EmoLLM}
& \textbf{92.1} & \textbf{1.87} & \textbf{4.86} & \textbf{4.92}
& \textbf{83.2} & \textbf{2.86} & \textbf{4.17} & \textbf{3.71}
& \textbf{95.3} & \textbf{2.18} & \textbf{4.08} & \textbf{4.59}
& \textbf{96.2} & \textbf{1.21} & \textbf{4.21} & \textbf{3.95} \\
\rowcolor{green!12}
$\Delta$ vs.\ Base
& $+3.8\%$  & $-38.3\%$ & $+6.1\%$  & $+28.1\%$
& $+7.2\%$  & $-12.5\%$ & $+1.0\%$  & $+8.8\%$
& $+6.8\%$  & $-25.3\%$ & $+3.3\%$  & $+28.6\%$
& $+15.3\%$ & $-26.2\%$ & $+11.1\%$ & $+3.1\%$ \\
\bottomrule
\end{tabular}%
}
\label{tab:qwen_results_grouped}
\end{table*}

\begin{figure*}[ht]
\centering
\begin{tikzpicture}
\begin{groupplot}[
    group style={group size=4 by 1, horizontal sep=0.55cm},
    width=0.26\textwidth,
    height=3.6cm,
    ymin=0.0, ymax=0.95,   
    clip=false,            
    xmin=0.4, xmax=4.6,
    xtick={1,2,3,4},
    xticklabels={EmoLLM, PERM, RLVER, Base},
    xticklabel style={rotate=28, anchor=east, font=\scriptsize},
    ymajorgrids=true,
    grid style={densely dotted, gray!35},
    tick label style={font=\footnotesize},
    title style={font=\small, yshift=-3pt},
    nodes near coords,
    nodes near coords style={
        font=\scriptsize,
        text=black,
        yshift=-2pt, 
        /pgf/number format/fixed,
        /pgf/number format/precision=2
    },
    every axis plot/.append style={
        ybar,
        bar width=13pt,
        bar shift=0pt,
        draw=black!70,
        line width=0.6pt
    }
]
\nextgroupplot[
    title={ED},
    ylabel={Avg. EG / Turn},
    ylabel style={font=\scriptsize},
]
\addplot[fill=red!70] coordinates {(1,0.5500)};
\addplot[fill=gray!55] coordinates {(2,0.5210)};
\addplot[fill=gray!70] coordinates {(3,0.4870)};
\addplot[fill=white] coordinates {(4,0.4500)};

\nextgroupplot[
    title={MSD},
    yticklabels={}
]
\addplot[fill=red!70] coordinates {(1,0.4000)};
\addplot[fill=gray!55] coordinates {(2,0.3740)};
\addplot[fill=gray!70] coordinates {(3,0.3380)};
\addplot[fill=white] coordinates {(4,0.3167)};

\nextgroupplot[
    title={MedD},
    yticklabels={}
]
\addplot[fill=red!70] coordinates {(1,0.4500)};
\addplot[fill=gray!55] coordinates {(2,0.4310)};
\addplot[fill=gray!70] coordinates {(3,0.4070)};
\addplot[fill=white] coordinates {(4,0.3833)};

\nextgroupplot[
    title={ICLR},
    yticklabels={}
]
\addplot[fill=red!70] coordinates {(1,0.7917)};
\addplot[fill=gray!55] coordinates {(2,0.7460)};
\addplot[fill=gray!70] coordinates {(3,0.7030)};
\addplot[fill=white] coordinates {(4,0.6667)};

\end{groupplot}
\end{tikzpicture}

\caption{Average Emotional Gain per Turn (EG/Turn) across four benchmarks. Higher values indicate greater positive emotional improvement during the dialogue.}
\label{fig:four_panel_bar}
\end{figure*}
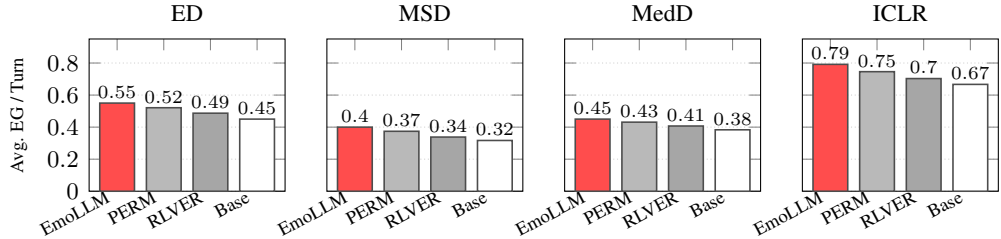

\begin{table*}[ht]
\caption{Comparison with proprietary LLM APIs across four datasets.}
\centering
\small
\setlength{\tabcolsep}{6.0pt}
\renewcommand{\arraystretch}{1.02} 
\newcommand{\apicolw}{2.7cm}       
\resizebox{\textwidth}{!}{%
\begin{tabular}{p{\apicolw}|cccc|cccc|cccc|cccc}
\toprule
\textbf{Model}
& \multicolumn{4}{c|}{\textbf{ED}}
& \multicolumn{4}{c|}{\textbf{MSD}}
& \multicolumn{4}{c|}{\textbf{MedD}}
& \multicolumn{4}{c}{\textbf{ICLR}} \\
& SR$\uparrow$ & AT$\downarrow$ & ES$\uparrow$ & EA$\uparrow$
& SR$\uparrow$ & AT$\downarrow$ & ES$\uparrow$ & EA$\uparrow$
& SR$\uparrow$ & AT$\downarrow$ & ES$\uparrow$ & EA$\uparrow$
& SR$\uparrow$ & AT$\downarrow$ & ES$\uparrow$ & EA$\uparrow$ \\
\midrule
\midrule
\multicolumn{17}{c}{\textit{Proprietary LLM APIs}} \\
\midrule
gpt-5-nano              
& 78.4 & 2.87 & 4.41 & 4.06 
& 74.2 & 3.84 & 4.12 & 3.57 
& 77.6 & 2.73 & 4.19 & 3.38 
& 78.1 & 2.61 & 4.08 & 3.02 \\
gpt-5-mini              
& 90.3 & 3.08 & 3.63 & 4.47 
& 81.4 & 2.97 & 4.01 & 3.53
& 88.2 & 2.15 & 4.06 & 4.28
& 92.1 & 2.08 & 3.91 & 3.57 \\
gemini-2.5-flash 
& 86.8 & 2.23 & 3.49 & 4.69
& 85.3 & 2.68 & 4.23 & 3.76
& 93.6 & 1.62 & 3.98 & 4.43
& 96.7 & 1.73 & 3.83 & 3.98 \\
gemini-3.1-flash-lite
& 79.2 & 2.49 & 4.28 & 3.91
& 76.4 & 2.62 & 4.11 & 3.46
& 81.3 & 1.69 & 3.86 & 4.68
& 83.1 & 1.32 & 3.88 & 3.49 \\
\midrule
\multicolumn{17}{c}{\textit{Ours (open-weight model)}} \\
\midrule
\rowcolor{green!12}
\textbf{EmoLLM}
& \textbf{92.1} & \textbf{1.87} & \textbf{4.86} & \textbf{4.92}
& \textbf{83.2} & \textbf{2.86} & \textbf{4.17} & \textbf{3.71}
& \textbf{95.3} & \textbf{2.18} & \textbf{4.08} & \textbf{4.59}
& \textbf{96.2} & \textbf{1.21} & \textbf{4.21} & \textbf{3.95} \\
\bottomrule
\end{tabular}%
}
\label{tab:api_comparison}
\end{table*}

\section{Experimental Setup}

\noindent \textbf{Datasets and Scenarios.}
Training and evaluation of EmoLLM are conducted on dialogue datasets covering emotional support, technical assistance, medical consultation, and academic peer review. For Stage~I CPT, the ESConv corpus~\cite{liu2021towards} is used and augmented with affective knowledge from the ECoK graph~\cite{wang2024ecok}. For supervised initialization and Stage~II, EmpatheticDialogues~\cite{rashkin2019towards}, MSDialog~\cite{qu2018analyzing}, MedDialog~\cite{he2020meddialog}, and a collection of ICLR peer-review comments are used. To enable controlled multi-turn evaluation, \emph{scenario seeds} are constructed from these datasets, where each seed consists of a short scenario description and an initial user utterance. Additional details on the datasets and scenario seeds are provided in Appendix~\ref{appendix:datasets}.

\noindent \textbf{Base Model.}
Experiments are conducted using Qwen3-8B~\cite{yang2025qwen3} as the initialization of EmoLLM. Generalization to LLaMA-3.1-8B~\cite{grattafiori2024llama} is reported in Appendix~\ref{appendix:backbone_generalization}.

\noindent \textbf{Training Procedure.}
Implementation details of the training procedure are provided in Appendix~\ref{appendix:training}.

\noindent \textbf{Baselines.}
EmoLLM is compared with representative baselines based on prompting, supervised fine-tuning, and reinforcement learning for emotional dialogue, as well as proprietary LLMs. Detailed descriptions of the baselines are provided in Appendix~\ref{appendix:baselines}.

\noindent \textbf{Evaluation.}
Evaluation of EmoLLM is conducted in a multi-turn role-play environment initialized from the constructed scenario seeds. Six metrics are reported: Factual Accuracy (FA) for cognitive reliability (IQ), and Success Rate (SR~(\%)), Average Turns (AT), Emotional State Score (ES), Emotional Gain per Turn (EG/Turn), and Empathic Appropriateness (EA) for emotional intelligence (EQ). Full definitions of the metrics and evaluation protocols are provided in Appendix~\ref{appendix:evaluation}.

\section{Experimental Results \& Analysis}

\subsection{Main Results}

\paragraph{Comparison with Baselines.} We evaluate EmoLLM against a diverse set of prompting, supervised, planning-based, and RL-based baselines across four datasets (ED, MSD, MedD, and ICLR). For the main comparison, we report Success Rate (SR), Average Turns (AT), Emotional State Score (ES), and Empathic Appropriateness (EA) in Table~\ref{tab:qwen_results_grouped}, and additionally summarize Emotional Gain per Turn (EG/Turn) in Figure~\ref{fig:four_panel_bar}. Together, these metrics capture emotional effectiveness, conversational efficiency, response appropriateness, and trajectory-level emotional improvement. As shown in Table~\ref{tab:qwen_results_grouped} and Figure~\ref{fig:four_panel_bar}, EmoLLM consistently outperforms existing approaches across all benchmarks. It achieves higher SR while requiring fewer turns to reach success, and also yields stronger ES, EA, and EG/Turn, indicating more effective and emotionally grounded interaction dynamics. These results collectively show that explicit cognitive--emotional co-reasoning improves both task success and interaction stability.

\paragraph{Comparison with Proprietary APIs.} To contextualize EmoLLM’s performance, we compare it against frontier proprietary LLM APIs under the same multi-turn role-play evaluation protocol. All models interact with an identical user simulator and scenario set, and are evaluated using the same four metrics: Success Rate (SR), Average Turns (AT), Emotional State Score (ES), and Empathic Appropriateness (EA). As summarized in Table~\ref{tab:api_comparison}, EmoLLM achieves performance comparable to, and in several metrics exceeding, strong proprietary APIs despite its significantly smaller parameter scale. These findings highlight that explicitly structured cognitive--emotional co-reasoning is a highly effective paradigm for equipping compact, open-weight models with the emotional intelligence typically associated with massive frontier systems.

\subsection{Effect of Reverse-Perspective Lookahead Depth}
\label{subsec:ablation_rp_steps}

\begin{figure}[ht]
\centering
\begin{tikzpicture}
\begin{axis}[
    width=0.8\linewidth,
    height=4.2cm,
    axis y line*=left,
    axis x line*=bottom,
    xmin=-0.1, xmax=4.1,
    ymin=75, ymax=92,
    xtick={0,1,2,3,4},
    xlabel={Reverse-perspective lookahead depth ($n$)},
    ylabel={SR (\%)},
    ymajorgrids=true,
    grid style={dashed, gray!25},
    tick label style={font=\scriptsize},
    label style={font=\scriptsize},
    legend style={
        font=\scriptsize,
        draw=none,
        fill=none,
        at={(0.5,1.12)},
        anchor=south,
        legend columns=4,
        column sep=0.9em,
    },
    legend cell align=left,
]
\addplot+[color=blue, mark=o, thick] coordinates
{(0,88) (1,90) (2,91) (3,88) (4,89)};
\addlegendentry{SR$\uparrow$}
\addlegendimage{color=purple!80!black, mark=square, thick}
\addlegendentry{ES$\uparrow$}
\addlegendimage{color=orange!85!black, mark=diamond, thick}
\addlegendentry{EA$\uparrow$}
\addlegendimage{color=brown!80!black, mark=triangle, thick}
\addlegendentry{AT$\downarrow$}
\end{axis}
\begin{axis}[
    width=0.8\linewidth,
    height=4.2cm,
    axis y line*=right,
    axis x line=none,
    xmin=-0.1, xmax=4.1,
    ymin=1.5, ymax=4.2,
    ytick={1.5,2.0,2.5,3.0,3.5,4.0},
    ylabel={ES/EA (1--5) \;/\; AT (turns)},
    tick label style={font=\scriptsize},
    label style={font=\scriptsize},
]
\addplot+[color=purple!80!black, mark=square, thick] coordinates
{(0,3.43) (1,3.67) (2,3.84) (3,3.68) (4,3.78)};
\addplot+[color=orange!85!black, mark=diamond, thick] coordinates
{(0,2.53) (1,2.97) (2,2.83) (3,2.89) (4,2.94)};
\addplot+[color=brown!80!black, mark=triangle, thick] coordinates
{(0,2.88) (1,2.44) (2,2.38) (3,2.40) (4,2.17)};
\end{axis}
\end{tikzpicture}
\caption{Effect of reverse-perspective lookahead depth on EmpatheticDialogues. SR is shown on the left axis; ES/EA (rated on a 1--5 scale) and AT (turns; lower is better) are shown on the right axis.}
\label{fig:stepN}
\end{figure}
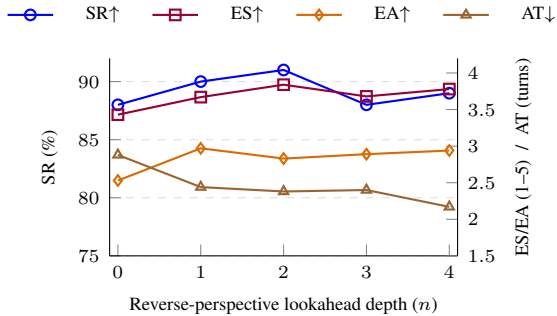

As illustrated in Figure~\ref{fig:stepN}, increasing the lookahead depth from $n=1$ to $n=3$ consistently improves Success Rate (SR) and Empathic Appropriateness (EA). This trend aligns with Theorem~\ref{thm:rollout_bias}, suggesting that deeper reasoning reduces truncation bias by capturing longer-term emotional transitions. However, gains diminish beyond $n=3$, as compounding simulation errors in $\mathcal{T}_\theta$ offset the benefit of longer horizons. This trade-off indicates that $n=3$ provides a favorable balance between reduced evaluation bias and increased model-based variance in our setting.

\subsection{Factual Accuracy after Training}

\begin{table}[ht]
\caption{
Factual Accuracy (FA; \%) across four datasets. ``w/o RL'' denotes Stage~I initialization without reinforcement learning. $\Delta$ indicates the relative improvement over the base model.
}
\centering
\scriptsize
\setlength{\tabcolsep}{6.0pt}
\renewcommand{\arraystretch}{1.05}
\resizebox{\linewidth}{!}{%
\begin{tabular}{lcccc}
\toprule
\textbf{Model} 
& \textbf{ED} 
& \textbf{MSD} 
& \textbf{MedD} 
& \textbf{ICLR} \\
\midrule
\rowcolor{gray!15}
Qwen3-8B (Base) 
& 73.4 & 94.4 & 87.5 & 83.8 \\
w/o RL 
& 72.3 & 87.5 & 76.1 & 78.9 \\
\midrule
\textbf{EmoLLM} 
& \textbf{80.3} & \textbf{95.4} & \textbf{91.9} & \textbf{86.7} \\
\rowcolor{green!8}
$\Delta$ vs.\ Base
& $+9.4\%$ & $+1.1\%$ & $+5.0\%$ & $+3.5\%$ \\
\bottomrule
\end{tabular}
}
\label{tab:fa_results}
\end{table}

We further examine whether improving emotional intelligence compromises cognitive reliability by measuring Factual Accuracy (FA). As shown in Table~\ref{tab:fa_results}, Stage~I initialization alone (w/o RL) causes noticeable drops in FA on several datasets, suggesting that structured initialization alone is insufficient to maintain stable factual grounding. In contrast, the full EmoLLM consistently improves FA over the base model (Qwen3-8B), with relative gains of +9.4\% on ED, +1.1\% on MSD, +5.0\% on MedD, and +3.5\% on ICLR. Part of this improvement is expected, since FA is explicitly encouraged during reinforcement learning. At the same time, the result is consistent with the design of the framework: ARG promotes factual grounding, while reverse-perspective feedback filters out responses that are emotionally appropriate but weakly grounded in context. These findings suggest that emotional alignment in EmoLLM does not compromise cognitive reliability and can even improve it in multi-turn interactions.

\subsection{Ablation Study}

\begin{table*}[ht]
\caption{
Ablation study of EmoLLM across four datasets.
``w/o ARG'' removes the Appraisal Reasoning Graph structure;
``w/o Stage-I (CPT+SFT)'' removes the Stage~I initialization consisting of knowledge-guided continued pretraining (CPT) and supervised ARG alignment (SFT);
``w/o RP Rewards'' removes the reverse-perspective rewards used during reinforcement learning.
}
\centering
\small
\setlength{\tabcolsep}{6.0pt}
\renewcommand{\arraystretch}{1.05}
\newcommand{\ablcolw}{3.0cm}
\resizebox{\textwidth}{!}{%
\begin{tabular}{p{\ablcolw}|cccc|cccc|cccc|cccc}
\toprule
\textbf{Variant}
& \multicolumn{4}{c|}{\textbf{ED}}
& \multicolumn{4}{c|}{\textbf{MSD}}
& \multicolumn{4}{c|}{\textbf{MedD}}
& \multicolumn{4}{c}{\textbf{ICLR}} \\
& SR$\uparrow$ & AT$\downarrow$ & ES$\uparrow$ & EA$\uparrow$
& SR$\uparrow$ & AT$\downarrow$ & ES$\uparrow$ & EA$\uparrow$
& SR$\uparrow$ & AT$\downarrow$ & ES$\uparrow$ & EA$\uparrow$
& SR$\uparrow$ & AT$\downarrow$ & ES$\uparrow$ & EA$\uparrow$ \\
\midrule
\rowcolor{green!12}
\textbf{EmoLLM (Full)}
& \textbf{92.1} & \textbf{1.87} & \textbf{4.86} & \textbf{4.92}
& \textbf{83.2} & \textbf{2.86} & \textbf{4.17} & \textbf{3.71}
& \textbf{95.3} & \textbf{2.18} & \textbf{4.08} & \textbf{4.59}
& \textbf{96.2} & \textbf{1.21} & \textbf{4.21} & \textbf{3.95} \\
\midrule
w/o ARG
& 86.4 & 2.06 & 4.52 & 4.57
& 72.3 & 3.08 & 3.91 & 3.37
& 89.8 & 2.43 & 3.81 & 4.17
& 92.1 & 1.39 & 3.90 & 3.68 \\
w/o Stage-I (CPT+SFT)
& 90.2 & 1.93 & 4.73 & 4.79
& 76.4 & 2.97 & 4.05 & 3.58
& 93.1 & 2.27 & 3.97 & 4.39
& 94.3 & 1.29 & 4.06 & 3.83 \\
w/o RP Rewards
& 87.2 & 1.81 & 4.61 & 4.67
& 73.1 & 2.72 & 3.89 & 3.41
& 91.0 & 2.11 & 3.83 & 4.21
& 93.2 & 1.19 & 3.99 & 3.73 \\
\bottomrule
\end{tabular}%
}
\label{tab:ablation}
\end{table*}

We conduct an ablation study to examine the contribution of EmoLLM’s three core components: the ARG, the Stage~I initialization, and the RP rewards used in reinforcement learning. As shown in Table~\ref{tab:ablation}, removing ARG consistently reduces Emotional State (ES) and Empathic Appropriateness (EA), highlighting the importance of explicit cognitive--emotional structure. Removing Stage~I initialization leads to broader degradation across metrics, indicating that early alignment provides an important foundation for later emotional learning. In contrast, removing RP rewards mainly weakens interaction-level performance, most notably lowering Success Rate (SR) and overall multi-turn effectiveness. Average Turns (AT) should be interpreted jointly with SR, since it is computed over successful episodes only. Taken together, these results suggest that EmoLLM benefits from the complementary integration of structured appraisal reasoning, staged optimization, and perspective-aware reinforcement learning.

\section{Related Work}
\label{sec:related}

\subsection{Emotional Dialogue in LLMs}

A growing body of work studies emotional intelligence in LLMs, including emotion recognition, empathetic response generation, and affect-aware interaction~\cite{wang2023emotional,li2023large,sorin2024large}. Early approaches rely on supervised fine-tuning on emotion-annotated dialogue corpora such as EmpatheticDialogues~\cite{rashkin2019towards}, DailyDialog~\cite{li2017dailydialog}, and MELD~\cite{poria2019meld}. Beyond direct imitation, prior work explores emotional dialogue through prompting, instruction tuning, external affective knowledge, and user-state modeling~\cite{zheng2023building,lee2023chain,chen2023empowering,zhang2024empathetic,cao2025tool,sarangi2025decompose}. More recent work studies emotional reasoning through cognitive appraisal and perspective-aware inference rather than surface-level style adaptation~\cite{hong2025third,yeo2025beyond}. Benchmarks such as TEA-Bench~\cite{sui2026tea}, HEART~\cite{iyer2026heart}, and CogToM~\cite{tong2026cogtom} further enable systematic evaluation. However, most existing methods focus on response-level improvements or isolated user-state inference, without modeling how situational facts, inferred needs, appraisal, emotion, and response strategy interact across multi-turn dialogue.

\subsection{Reinforcement Learning for LLMs EQ}

RL is widely used to align LLMs with human preferences and desired behaviors, including RLHF~\cite{ouyang2022training,bai2022training}, RLAIF~\cite{lee2023rlaif}, and optimization methods such as PPO, DPO, and GRPO~\cite{schulman2017proximal,rafailov2023direct,shao2024deepseekmath}. Recent work extends RL to emotional and social settings, including empathy-oriented dialogue agents~\cite{jhan2021cheerbots}, value-based emotional alignment~\cite{rakib2025dialogxpert}, verifiable emotion rewards~\cite{wang2025rlverreinforcementlearningverifiable}, psychology-grounded reward modeling~\cite{wang2026perm}, and future-oriented emotional support rewards~\cite{yang2025towards,zhang2025echo}. While these methods improve emotional outcomes, most reward signals remain heuristic or response-level, providing limited supervision for coherent multi-turn cognitive--emotional reasoning. In contrast, EmoLLM uses RL to refine an appraisal-grounded structure linking situational facts, user-state inference, emotion, and response strategy.

\section{Conclusion}
\label{sec:conclusion}

We introduce EmoLLM, an appraisal-grounded framework for emotionally intelligent multi-turn dialogue based on explicit IQ--EQ co-reasoning. With the Appraisal Reasoning Graph as a structured reasoning scaffold and reverse-perspective reinforcement learning with explicit lookahead as the optimization mechanism, EmoLLM evaluates candidate responses through their predicted downstream effects on the user’s appraisal and emotional state. Across diverse dialogue benchmarks, EmoLLM improves multi-turn interaction quality, conversational efficiency, and factual reliability. These findings highlight the value of explicit cognitive--emotional modeling and point to a promising direction for developing conversational agents with more tightly integrated IQ and EQ.

\section*{Limitations}

Our study has several limitations. Most automatic metrics rely on LLM-based evaluators, which may introduce model-dependent biases despite consistent evaluation protocols. In addition, Stage~II training is conducted in a simulated multi-turn environment rather than with real human users, which may not fully capture the diversity of real-world emotional interactions. While we evaluate across multiple dialogue domains, further evaluation in more diverse real-world settings would strengthen the generality of our findings. Finally, the structured reasoning traces produced by EmoLLM should be viewed as task-oriented intermediate scaffolds rather than guaranteed faithful explanations of the model’s internal decision process.

\section*{Ethical Considerations}

EmoLLM is designed for research on emotionally intelligent dialogue and should not be viewed as a replacement for professional support in sensitive domains such as mental health or medical consultation. Although the framework aims to improve emotional appropriateness, model-generated responses may still be incomplete or misleading. Because systems that model user emotions may create risks of over-trust or unintended emotional influence, they should be deployed with transparency, clear limitations, and appropriate human oversight, particularly in high-stakes settings.




\bibliography{custom}

\appendix

\section{Reward Rubrics}
\label{appendix:reward_rubrics}

This section summarizes the judging criteria for the turn-level rewards and the trajectory-level emotional reward used in Stage~II training.

\paragraph{Cognitive reliability reward $R_{\text{cog}}(t)$.}
Evaluates whether the response $y_t$ is factually and logically consistent with the dialogue context $x_{\le t}$. Higher rewards are assigned to grounded, context-consistent responses, while hallucinations, contradictions, and unsupported claims receive lower scores.

\paragraph{ARG trace quality reward $R_{\text{ARG}}(t)$.}
Evaluates whether the structured reasoning trace is well-formed and semantically coherent. This includes grounded facts, justified needs, plausible appraisals, emotion--appraisal consistency, and strategy relevance. Lower rewards are given for unsupported inferences or internally inconsistent reasoning chains.

\paragraph{Reverse-perspective reward $R_{\mathrm{RP}}(t)$.}
Computed as $R_{\mathrm{RP}}(t)=\mathcal{S}(\hat{s}_t^{(n)})$ based on the predicted user-side transition. Unlike $R_{\text{ARG}}(t)$, which evaluates the assistant's internal reasoning, $R_{\mathrm{RP}}(t)$ evaluates whether the predicted user-side needs--appraisal--emotion update is plausible, contextually consistent, and aligned with the chosen response strategy.

\paragraph{Overthinking penalty $R_{\text{overthink}}(t)$.}
Penalizes redundant computation to encourage efficient reasoning. Higher penalties are applied when the model explicitly invokes an ARG trace for simple, straightforward queries, or when it performs excessive dynamic $n$-step lookahead that provides little additional predictive value.

\paragraph{Trajectory-level emotional reward $R_{\text{emo}}(\tau)$.}
Measures the net improvement in the user’s predicted emotional state over the full trajectory. Computed as the difference between the predicted initial and final emotional-state scores, where larger values indicate greater overall emotional improvement.

Together, these rewards encourage responses that are cognitively reliable, structurally coherent, and beneficial to the predicted user-side emotional trajectory.

\section{Details of Stage II Training}
\label{appendix:stage2_details}

\subsection{Reverse-Perspective Reasoning Details}
\label{appendix:details_rp}

This subsection provides additional implementation details for the reverse-perspective reasoning module used in Stage~II. Its role is not to change the external response format, but to provide an internal estimate of the user-side consequence of a candidate response for reward evaluation.

\paragraph{User-role transition prediction.}
At turn $t$, after generating response $y_t$, the model is prompted in a user role to predict the user-side transition induced by $y_t$ given the dialogue history $x_{\le t}$. The prediction target is a compact latent summary
\[
\hat{s}_t^{(n)}=(\hat{N}_t,\hat{A}_t,\hat{\mathcal{E}}_t),
\]
which represents the predicted post-response update in the user's needs/goals, appraisals, and emotional state.

\paragraph{Iterative lookahead depth.}
The transition predictor is applied iteratively for a small number of steps,
\[
\hat{s}_t^{(1)}=\mathcal{T}_\theta(x_{\le t},y_t),  \;
\hat{s}_t^{(k+1)}=\mathcal{T}_\theta(x_{\le t},y_t,\hat{s}_t^{(k)}),
\]
to approximate short-horizon user-side consequence estimation. In practice, we use shallow lookahead with $n\in\{1,2,3\}$, which is sufficient to capture immediate and near-future user-side effects while limiting error accumulation in the learned predictor.

\paragraph{Judge-based scoring.}
A judge model evaluates the predicted transition $\hat{s}_t^{(n)}$ and produces the reverse-perspective reward $R_{\mathrm{RP}}(t)$. The scoring rubric considers: (i) contextual consistency with $x_{\le t}$; (ii) psychological plausibility of the inferred needs, appraisals, and emotional state; and (iii) consistency between the predicted user-side consequence and the selected response strategy. This score is used as one component of the Stage~II reward.

\subsection{Reward Aggregation and GRPO Objective}
\label{appendix:grpo_objective}

For a trajectory $\tau=\{(x_t,z_t^{\mathrm{ARG}},y_t)\}_{t=1}^T$, we first define the turn-level reward at step $t$ as
\begin{equation*}
\begin{aligned}
r_t
&=
\lambda_{\text{cog}} R_{\text{cog}}(t)
+
\lambda_{\text{ARG}} R_{\text{ARG}}(t) \\
&\quad
+
\lambda_{\mathrm{RP}} R_{\mathrm{RP}}(t)
-
\lambda_{\text{over}} R_{\text{overthink}}(t).
\end{aligned}
\end{equation*}
where the $\lambda$'s are nonnegative weighting coefficients. We then combine turn-level rewards with the trajectory-level emotional reward $R_{\text{emo}}(\tau)$ to obtain the overall trajectory score
\begin{equation}
R(\tau)
=
\sum_{t=1}^{T} r_t
+
\lambda_{\text{emo}} R_{\text{emo}}(\tau).
\end{equation}

Following GRPO, for each prompt or initial dialogue context, we sample a group of $G$ trajectories
\[
\{\tau_i\}_{i=1}^{G} \sim \pi_{\theta_{\mathrm{old}}}(\cdot \mid x),
\]
and compute their normalized relative advantages as
\begin{equation*}
\mu_R
=
\frac{1}{G}\sum_{i=1}^{G} R(\tau_i), \;
\sigma_R
=
\sqrt{
\frac{1}{G}\sum_{i=1}^{G}
\bigl(R(\tau_i)-\mu_R\bigr)^2
}.
\end{equation*}

\begin{equation*}
A_i
=
\frac{R(\tau_i)-\mu_R}{\sigma_R+\epsilon}.
\end{equation*}

The policy is then optimized with the GRPO objective
\begin{equation*}
\begin{aligned}
&\mathcal{L}_{\mathrm{GRPO}}(\theta)
=
\mathbb{E}_{x,\{\tau_i\}_{i=1}^{G}}
\Biggl[
\frac{1}{G}\sum_{i=1}^{G}
\ell_i(\theta)
\Biggr],
\\
\ell_i(\theta)
&=
\min\Bigl(
\rho_i(\theta) A_i,\;
\mathrm{clip}\bigl(\rho_i(\theta),1-\varepsilon,1+\varepsilon\bigr) A_i
\Bigr).
\end{aligned}
\end{equation*}
where
\begin{equation}
\rho_i(\theta)
=
\prod_{t=1}^{T_i}
\frac{\pi_{\theta}(y_{i,t}\mid x_{i,\le t})}
{\pi_{\theta_{\mathrm{old}}}(y_{i,t}\mid x_{i,\le t})}.
\end{equation}
In practice, we optimize the token-level implementation of this objective as in standard RLHF/RLVR training, while the above trajectory-level form highlights how cognitive, appraisal, reverse-perspective, and emotional signals are jointly incorporated.

\subsection{Proof of Theorem~\ref{thm:rollout_bias}}
\label{appendix:theory_rp_lookahead}

In this section, we prove Theorem~\ref{thm:rollout_bias} under the latent-state discounted MDP abstraction introduced in Section~\ref{sec:rl_roleplay}. The result isolates the truncation bias inherent in finite-depth simulations.

\paragraph{Setup.}
Let the true action-value function under policy $\pi$ be
\begin{equation*}
Q_{\pi}(s,a)
=
\mathbb{E}\!\left[
\sum_{k=0}^{\infty} \gamma^k r(s_k, a_k)
\,\middle|\,
s_0=s,\; a_0=a
\right],
\end{equation*}
and let its $n$-step truncated lookahead estimate be
\begin{equation*}
Q_{\pi}^{(n)}(s,a)
=
\mathbb{E}\!\left[
\sum_{k=0}^{n-1} \gamma^k r(s_k, a_k)
\,\middle|\,
s_0=s,\; a_0=a
\right].
\end{equation*}

\begin{proof}
By decomposing the infinite discounted return into its first $n$ terms and the remaining tail, we have
\begin{equation*}
\begin{aligned}
Q_{\pi}(s,a)-Q_{\pi}^{(n)}(s,a)
&=
\mathbb{E}\!\left[
\sum_{k=n}^{\infty} \gamma^k r(s_k, a_k)
\right. \\
&\qquad\left.
\,\middle|\,
s_0=s,\; a_0=a
\right].
\end{aligned}
\end{equation*}
Taking absolute values and using the triangle inequality gives
\begin{equation*}
\begin{aligned}
&\bigl|Q_{\pi}(s,a)-Q_{\pi}^{(n)}(s,a)\bigr| \\
&\le
\mathbb{E}\!\left[
\sum_{k=n}^{\infty} \gamma^k |r(s_k, a_k)|
\,\middle|\,
s_0=s,\; a_0=a
\right] \\
&\le \sum_{k=n}^{\infty} \gamma^k R_{\max} \\
&=\gamma^n \sum_{j=0}^{\infty} \gamma^j R_{\max}
=\frac{\gamma^n}{1-\gamma} R_{\max}.
\end{aligned}
\end{equation*}
This proves the claim.
\end{proof}

\section{Dataset Sources and Scenario Preparation}
\label{appendix:datasets}

\begin{table*}[ht]
\caption{Dataset statistics for EmoLLM training and evaluation.}
\centering
\scriptsize
\begin{tabular}{lcccc}
\toprule
\textbf{Dataset} & \textbf{Stage} & \textbf{Conversations} & \textbf{Avg. Turns} & \textbf{Domain} \\
\midrule
ESConv & CPT & 1,300 & 10--15 & Emotional Support \\
ECoK & CPT & 24,000+ triples & -- & Knowledge Graph \\
\midrule
MSDialog & SFT/RL & 35,536 & 5--10 & Technical Support \\
EmpatheticDialogues & SFT/RL & 23,149 & 4--6 & Empathetic Dialogue \\
MedDialog & SFT/RL & 603 & 1 & Medical Consultation \\
ICLR Comments & SFT/RL & 18,942 & 3--5 & Academic Discourse \\
\bottomrule
\end{tabular}
\label{tab:dataset_stats}
\end{table*}

Table~\ref{tab:dataset_stats} summarizes datasets used across stages of \textbf{EmoLLM}. These span emotional support, technical assistance, medical consultation, and academic peer review, enabling evaluation of cognitive--emotional co-reasoning under diverse interaction settings.

\subsection{Stage I: Continued Pretraining Sources}

For Stage~I continued pretraining, ESConv~\cite{liu2021towards} is used as the primary emotional support corpus, augmented with appraisal-relevant knowledge from the ECoK graph~\cite{wang2024ecok}. ECoK provides structured affective triples as weak structural grounding signals during continued pretraining.

\subsection{Stage II: Role-Play Scenario Sources}

For ARG-guided supervised initialization and role-play reinforcement learning, EmpatheticDialogues~\cite{rashkin2019towards}, MSDialog~\cite{qu2018analyzing}, MedDialog~\cite{he2020meddialog}, and ICLR peer-review comments are used. For supervised initialization, \textbf{10\%} of training contexts are randomly sampled from each dataset and annotated with teacher-generated ARG traces and gating labels. For role-play training, dialogue examples are converted into compact scenario seeds summarizing the initial user situation and context, rather than replaying full conversations. Evaluation uses separately constructed held-out scenario sets in the same format, ensuring controlled and comparable multi-turn interactions without train--test leakage.

\subsection{Scenario Construction}

For multi-turn role-play evaluation, dialogue examples from test splits are transformed into unified scenario seeds. Instead of replaying full conversations, a compact scenario representation is extracted using the GPT API (\texttt{gpt-5-mini}), so each episode starts from a controlled yet open-ended initialization.

For each dataset, 100 dialogue instances are randomly sampled from the test split. Given a dialogue example, GPT summarizes the interaction into two fields: (1) \texttt{scenario}, describing the assistant role and user situation, and (2) \texttt{initial\_prompt}, the opening user utterance.

The resulting format is

\begin{lstlisting}
{
  "scenario": "You are an empathetic companion supporting someone who feels emotionally overloaded and drained.",
  "initial_prompt": "Everything feels piled up at once, and I cannot keep up anymore."
}
\end{lstlisting}

This abstraction preserves essential situational and emotional context while preventing leakage from future turns. All models therefore start from the same initial condition and must generate the subsequent interaction trajectory. The constructed scenario sets will be released as a benchmark for reproducible evaluation of multi-turn cognitive--emotional interactions.

\subsection{Role-Play Reformulation}

All datasets are converted into a unified scenario-based format for multi-turn simulation. For single-turn sources (e.g., MedDialog), minimal conversational scaffolds are constructed to preserve intent while enabling interactive role-play rollouts. For multi-turn sources, the original dialogue structure is retained when extracting scenario seeds and initializing context.

\section{Training Setup}
\label{appendix:training}

\subsection{Stage I: Initialization}

\paragraph{Knowledge-Grounded Continued Pretraining.}
In Stage~I CPT, dialogue contexts are augmented with appraisal-relevant knowledge retrieved from the ECoK graph and linearized as short reasoning blocks prepended to the input. The model is trained with a standard causal language modeling objective over the full sequence. We use LoRA~\cite{hu2022lora} ($r{=}16$, $\alpha{=}32$, dropout $0.1$) on the attention and MLP projections. Training uses AdamW with learning rate $1\times10^{-4}$, per-device batch size $16$, gradient accumulation $2$, one epoch, maximum sequence length $320$, $100$ warmup steps, and gradient clipping $1.0$. BF16 is used where supported, otherwise FP16.

\paragraph{Supervised Reasoning Initialization.}
For supervised initialization, we construct a teacher-generated instruction-tuning dataset of ARG-aligned assistant outputs using \texttt{gpt-5-mini}. Each instance contains the full dialogue history as context, and loss is applied only to the teacher-generated target, which may include either an ARG trace plus response or a response-only output. We use the same LoRA configuration as in CPT. Training uses AdamW with learning rate $2\times10^{-4}$, per-device batch size $2$, gradient accumulation $4$, three epochs, maximum sequence length $2{,}048$, $100$ warmup steps, and gradient clipping $1.0$.

\subsection{Stage II: Reinforcement Learning}

\paragraph{Reinforcement Learning Setup.}
Stage~II reinforcement learning is conducted in simulated multi-turn conversations of up to $8$ turns. A user simulator (\texttt{gpt-5-mini}) generates user responses conditioned on the evolving dialogue state. Policy optimization uses GRPO~\cite{shao2024deepseekmath} with KL coefficient $\beta{=}0.01$. At each turn, EmoLLM samples $G{=}4$ candidate responses (temperature $0.8$, top-$p{=}0.9$, top-$k{=}50$, repetition penalty $1.1$, up to $256$ new tokens) to form a group for relative reward estimation. These candidates are evaluated by an LLM-based judge (\texttt{gpt-5-mini}) using the reward components defined in Section~\ref{sec:rl_roleplay} and Appendix~\ref{appendix:reward_rubrics}. We train up to $50$ RL steps with a new LoRA adapter ($r{=}16$, $\alpha{=}32$, dropout $0.1$) on top of the merged Stage~I model, using AdamW with learning rate $1\times10^{-5}$ and gradient clipping $1.0$.

\subsection{Implementation Details}

All experiments use Qwen3-8B as the main backbone. The CPT and supervised initialization adapters are trained sequentially and merged into the base model before Stage~II. Context length is truncated to $2{,}048$ tokens during generation and to $896$ tokens during GRPO updates for memory efficiency. Training is implemented in PyTorch with HuggingFace Transformers and PEFT. All models are trained once, while evaluation results are averaged over three independent runs.

\section{Baselines}
\label{appendix:baselines}

We compare EmoLLM with representative baselines spanning four categories: the base model, prompting-based methods, supervised fine-tuning methods, and reinforcement learning methods. This grouping follows the main design paradigms used in emotionally aware dialogue modeling and matches the organization in Table~\ref{tab:qwen_results_grouped}.

\subsection{Base Model}

\noindent \textbf{Qwen3-8B.}
The instruction-tuned base language model is evaluated directly under standard prompting, without explicit emotional reasoning instructions or task-specific adaptation.

\subsection{Prompting Methods}

\noindent \textbf{Best-of-$N$.}
We sample multiple candidate responses from the base model and select the one with the highest evaluation score according to the same LLM-based evaluator used in our framework.

\noindent \textbf{ARG-Prompt.}
We prompt the base model to follow the ARG scaffold at inference time before generating the final response, isolating the effect of structured prompting without additional training.

\noindent \textbf{Ask-an-Expert}~\cite{zhang2023ask}.
An instruction-based prompting method that encourages the model to reason as a dialogue strategy expert before producing a response.

\noindent \textbf{ICL-AIF}~\cite{fu2023improving}.
An in-context multi-agent framework in which LLMs perform emotional role-play and an additional model provides iterative feedback for refinement.

\subsection{Supervised Fine-Tuning Methods}

\noindent \textbf{SFT.}
A standard supervised fine-tuning baseline using response-level supervision only, without explicit reasoning-trace supervision.

\noindent \textbf{ESCoT}~\cite{rashkin2019towards}.
A chain-of-thought supervision baseline that guides the model through emotion recognition, appraisal, and supportive strategy planning before response generation.

\noindent \textbf{PPDP}~\cite{deng2023plug}.
A strategy-planning framework that predicts dialogue strategies and improves responses using AI-feedback-driven optimization.

\subsection{Reinforcement Learning Methods}

\noindent \textbf{RLFF-ESC}~\cite{yang2025towards}.
A reinforcement learning baseline for emotional support dialogue that optimizes responses using rewards defined over simulated future interaction outcomes.

\noindent \textbf{RLVER}~\cite{wang2025rlverreinforcementlearningverifiable}.
A reinforcement learning method that uses verifiable emotion-related rewards derived from simulated user emotional states.

\noindent \textbf{PERM}~\cite{wang2026perm}.
A preference-based reinforcement learning method for emotional dialogue generation that optimizes responses using feedback on empathy and helpfulness quality.

\section{Evaluation Protocol and Metrics}
\label{appendix:evaluation}

\paragraph{Multi-turn role-play protocol.}
All evaluations are conducted in a multi-turn role-play setting under an \emph{LLM-as-a-Judge} framework.
A fixed LLM-based user simulator generates user replies and emotional transitions, while a separate LLM-based evaluator scores outcomes.
For each dataset, we construct scenario seeds specifying the initial user state, context, and implicit intent.
Each model is evaluated on the same scenarios with a fixed maximum number of turns.
An episode terminates when a predefined emotional success condition is met or the turn limit is reached.

\paragraph{Success Rate and Average Turns.}
At each turn, an LLM-based evaluator determines whether the dialogue reaches a satisfactory emotional resolution under the given scenario.
Success Rate (SR) is the percentage of episodes that reach such a state within the turn limit.
Average Turns (AT) measures the mean number of assistant turns required to reach success, computed over successful episodes.

\paragraph{Emotional State (ES).}
Emotional State (ES) measures the predicted user emotional state at the end of the dialogue.
An LLM-based evaluator assigns a trajectory-level score on a 1--5 Likert scale, where higher values correspond to more positive states (e.g., relief, comfort, or reassurance).
Given the predicted final state $e_T$, we report the average ES across episodes.
Higher ES indicates more positive user states at the end of the interaction.

\paragraph{Emotional Gain per Turn (EG/Turn).}
To quantify emotional improvement normalized by interaction length, we define Emotional Gain per Turn (EG/Turn).
Let $e_0$ and $e_T$ denote the predicted initial and final user states on the same 1--5 scale, and let $T$ denote the number of assistant turns until termination.
We compute $\mathrm{EG/Turn}=(e_T-e_0)/T$.
Higher values indicate larger improvement per turn.
We report EG/Turn when analyzing trajectory efficiency, such as in Figure~\ref{fig:four_panel_bar}.

\paragraph{Empathic Appropriateness.}
Empathic Appropriateness (EA) evaluates response quality in terms of empathy, tone, and contextual fit.
An LLM-based judge assigns a 1--5 Likert score to each response based on whether it acknowledges the user’s emotional state, provides supportive and relevant guidance, and avoids dismissive or inappropriate language.
EA scores are averaged across dialogue turns and episodes.

\paragraph{Factual Accuracy.}
Factual Accuracy (FA) measures whether responses remain correct on fact-sensitive turns, such as technical assistance or medical consultation.
An LLM-based evaluator checks for hallucinations, incorrect claims, or misleading suggestions against reference information or domain knowledge.
FA is reported as the percentage of factually correct responses.

\section{Additional Experiments}

\subsection{Response Quality Analysis}
\label{appendix:analysis}

Beyond task-level metrics, we further evaluate response quality along six human-centered dimensions relevant to emotional support dialogue: Empathy, Topic Guidance, Value Guidance, Intention Following, Fluency, and Colloquial Expression. We use an LLM-based evaluation protocol with a five-point Likert scale, where each response is scored independently on these dimensions under a fixed rubric and prompt template. Scores are averaged across sampled dialogue episodes to obtain model-level results. As shown in Figure~\ref{fig:radar_quality}, EmoLLM improves empathy- and guidance-related dimensions over the base model and the ARG-only variant, while maintaining strong fluency and natural conversational expression. These results suggest that structured appraisal reasoning helps the model better infer latent user needs and produce more supportive responses.

\definecolor{cBase}{HTML}{4C7A89}
\definecolor{cArg}{HTML}{B58A3A}
\definecolor{cEmo}{HTML}{D98C94}

\begin{figure}[t]
\centering
\begin{tikzpicture}
\begin{polaraxis}[
    width=0.6\linewidth,
    height=0.6\linewidth,
    ymin=0, ymax=1.2,
    ytick={0.3,0.6,0.9,1.2},
    yticklabels={3.5,4.0,4.5,5.0},
    yticklabel pos=right,
    yticklabel style={
        font=\scriptsize
    },
    major grid style={draw=gray!30},
    axis line style={draw=gray!40},
    axis on top,
    grid=major,
    minor tick num=0,
    major grid style={line width=0.5pt, draw=gray!25},
    axis line style={draw=gray!35},
    xtick={0,60,120,180,240,300},
    xticklabels={
        Empathy,
        Topic,
        Value,
        Intention\\Following,
        Fluency,
        Colloquial\\Expression
    },
    xticklabel style={
        font=\scriptsize\bfseries,
        align=center,
        text width=1.5cm
    },
    legend style={
        at={(0.5,1.15)},
        anchor=south,
        legend columns=3,
        draw=none,
        font=\scriptsize,
        column sep=6pt
    },
]

\addplot+[
    color=cBase,
    thick,
    mark=*,
    mark size=1.6pt,
    mark options={fill=white},
    fill=cBase,
    fill opacity=0.12
] coordinates {
    (0,0.66) (60,0.75) (120,0.57) (180,0.69) (240,0.93) (300,0.72) (0,0.66)
};
\addlegendentry{\textbf{Base}}

\addplot+[
    color=cArg,
    thick,
    mark=square*,
    mark size=1.6pt,
    mark options={fill=white},
    fill=cArg,
    fill opacity=0.12
] coordinates {
    (0,0.81) (60,0.93) (120,0.72) (180,0.78) (240,0.90) (300,0.69) (0,0.81)
};
\addlegendentry{\textbf{+ARG}}

\addplot+[
    color=cEmo,
    very thick,
    mark=triangle*,
    mark size=2.0pt,
    mark options={fill=white},
    fill=cEmo,
    fill opacity=0.14
] coordinates {
    (0,1.03) (60,1.01) (120,0.87) (180,0.85) (240,0.95) (300,0.82) (0,1.03)
};
\addlegendentry{\textbf{+ARG+RP}}

\end{polaraxis}
\end{tikzpicture}
\caption{Response-quality evaluation across six human-centered dimensions. Scores are measured on a 1--5 Likert scale and linearly mapped to the plotted radial range for visualization. EmoLLM improves empathy and guidance-related dimensions over both the base model and the ARG-only variant, while preserving strong fluency and naturalness.}
\label{fig:radar_quality}
\end{figure}
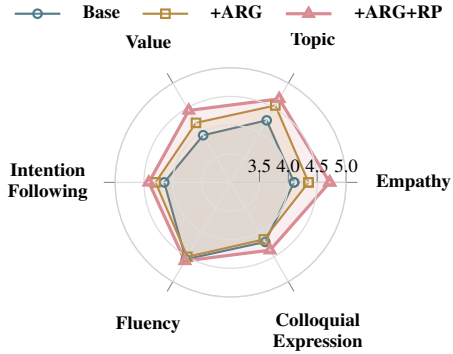

\subsection{Backbone Generalization}
\label{appendix:backbone_generalization}

\begin{table}[ht]
\caption{Backbone generalization results on LLaMA-3.1-8B. EmoLLM consistently improves emotional intelligence metrics (SR, ES, AT, and EA) while also yielding stronger factual accuracy (FA) across datasets.}
\centering
\scriptsize
\setlength{\tabcolsep}{5pt}
\renewcommand{\arraystretch}{1.12}
\begin{tabular}{llrrrrr}
\toprule
\textbf{Dataset} & \textbf{Model} & \textbf{SR (\%)} & \textbf{ES} & \textbf{AT} & \textbf{EA} & \textbf{FA (\%)} \\
\midrule
\multirow{2}{*}{ED}
& Base   & 71.4 & 4.18 & 3.46 & 3.21 & 44.7 \\
& EmoLLM & \textbf{95.6} & \textbf{4.37} & \textbf{1.94} & \textbf{4.45} & \textbf{77.9} \\
\cmidrule(lr){1-7}
\multirow{2}{*}{MSD}
& Base   & 88.7 & 3.86 & 2.73 & 2.52 & 52.4 \\
& EmoLLM & \textbf{94.8} & \textbf{4.05} & \textbf{1.36} & \textbf{4.02} & \textbf{66.8} \\
\cmidrule(lr){1-7}
\multirow{2}{*}{MedD}
& Base   & 84.9 & 3.71 & 1.57 & 3.39 & 61.8 \\
& EmoLLM & \textbf{91.2} & \textbf{3.92} & \textbf{1.24} & \textbf{4.06} & \textbf{76.4} \\
\cmidrule(lr){1-7}
\multirow{2}{*}{ICLR}
& Base   & 79.6 & 3.42 & 3.08 & 3.31 & 50.3 \\
& EmoLLM & \textbf{89.1} & \textbf{3.78} & \textbf{2.15} & \textbf{3.92} & \textbf{70.6} \\
\bottomrule
\end{tabular}
\label{tab:llama_generalization}
\end{table}

To assess backbone generalization, we replicate the EmoLLM training pipeline on LLaMA-3.1-8B~\cite{grattafiori2024llama}. 
Table~\ref{tab:llama_generalization} compares EmoLLM with the corresponding base model under the same evaluation protocol. EmoLLM consistently improves emotional intelligence and interaction quality across all benchmarks. These results indicate that the benefits of structured cognitive--emotional co-reasoning are not backbone-specific and transfer effectively to alternative open-weight LLMs.

\subsection{Efficiency Analysis}
\label{appendix:efficiency}

\begin{table}[ht]
\caption{Output-token overhead comparison between the base model and EmoLLM across the four evaluation datasets}
\centering
\setlength{\tabcolsep}{4pt}
\renewcommand{\arraystretch}{1.05}
\scriptsize
\begin{tabular}{lcccc}
\toprule
\multirow{2}{*}{\textbf{Dataset}} 
& \multicolumn{2}{c}{\textbf{Tokens / Turn}} 
& \multicolumn{2}{c}{\textbf{Tokens / Conversation}} \\
\cmidrule(lr){2-3} \cmidrule(lr){4-5}
& \textbf{Base} & \textbf{EmoLLM} & \textbf{Base} & \textbf{EmoLLM} \\
\midrule
ED   & 92  & 115 & 410 & 403 \\
MSD  & 105 & 127 & 520 & 534 \\
MedD & 118 & 138 & 610 & 581 \\
ICLR & 96  & 116 & 470 & 446 \\
\bottomrule
\end{tabular}
\label{tab:efficiency_tokens}
\end{table}

We analyze the efficiency of EmoLLM by comparing output tokens with the base model. We report average tokens per turn and per conversation across four datasets. As shown in Table~\ref{tab:efficiency_tokens}, EmoLLM produces moderately more tokens per turn due to selectively triggered co-reasoning traces. However, conversation-level overhead remains small and is sometimes even lower than the base model, because improved interaction efficiency (i.e., fewer turns to resolution; Table~\ref{tab:qwen_results_grouped}) partially offsets the added tokens at each turn. This is further aided by the compact ARG representation relative to free-form chain-of-thought.

\subsection{Manual Preference Check}
\label{appendix:human_eval}

\begin{table}[ht]
\centering
\caption{Manual ranking comparison between the base model, SFT, and EmoLLM. Lower average rank indicates better preference.}
\setlength{\tabcolsep}{6pt}
\renewcommand{\arraystretch}{1.1}
\scriptsize
\begin{tabular}{lccc}
\toprule
\textbf{Dataset} & \textbf{Base} & \textbf{SFT} & \textbf{EmoLLM} \\
\midrule
ED   & 2.52 & 2.01 & 1.47 \\
MSD  & 2.48 & 2.06 & 1.46 \\
MedD & 2.57 & 2.00 & 1.43 \\
ICLR & 2.44 & 2.08 & 1.48 \\
\bottomrule
\end{tabular}
\label{tab:human_ranking}
\end{table}

To complement the LLM-based evaluation, we conduct a small-scale blinded internal preference check comparing three systems: the base model (Qwen3-8B), the SFT model, and EmoLLM. For each dataset, we randomly sample 25 dialogue scenarios and generate conversation trajectories from all three systems under identical prompts, resulting in 100 comparisons in total. For each scenario, the three trajectories are presented in randomized order with system identities hidden, and are ranked from best to worst based on overall interaction quality, considering contextual appropriateness, emotional support, and response coherence. The reported score corresponds to the average rank of each system across the evaluated scenarios, where lower values indicate better preference. As shown in Table~\ref{tab:human_ranking}, EmoLLM consistently achieves the best average rank across all datasets. Because this check is small in scale and intended only for qualitative validation, it should be interpreted as supplementary evidence rather than a formal human evaluation.

\section{Prompt Templates and Multi-Turn Role-Play Environment}
\label{appendix:prompts}

We summarize representative prompt templates used for Stage~I supervision, Stage~II rollout, user simulation, and evaluation. The templates below illustrate the core formats used in our implementation; minor wording variations may be used across stages or experimental settings. Prompts with explicit reverse-perspective fields are used only during training-time rollout or controlled comparison. At inference time, EmoLLM generates responses directly without explicitly instantiating these fields.

\subsection*{Stage I: Supervised Initialization}

\noindent\textbf{SFT Training Target Format.}

\begin{lstlisting}
Input (dialogue context):
User: ...
Assistant: ...
User: ...
...

Target (assistant output):
<think>
[Contextual Facts] ...
[Inferred Needs and Goals] ...
[Appraisal] ...
[Emotional State] ...
[Response Strategy] ...
[Optional Reverse-Perspective Step 1] ...        (optional in a subset of teacher-generated samples)
[Optional Reverse-Perspective Step 2] ...        (optional in a subset of teacher-generated samples)
[Optional Reverse-Perspective Step 3] ...        (optional in a subset of teacher-generated samples)
[Optional Strategy Revision Based on Reverse-Perspective Steps] ...   (optional in a subset of teacher-generated samples)
</think>

<response>
...
</response>
\end{lstlisting}

\noindent
The optional reverse-perspective fields above are teacher-side augmentation and are not part of the core ARG state definition.

\noindent\textbf{ARG Annotation Prompt.}

\begin{lstlisting}
You are an empathetic and thoughtful conversational assistant.
Given a multi-turn conversation history and a current user message,
generate an ARG-aligned reasoning trace and a response.

Conversation History:
<prior dialogue turns or "(No previous conversation)">

Current User Message:
<current user utterance>

Please provide your answer in the following EXACT format.
The optional reverse-perspective sections should be included only when
they provide useful additional reasoning.

<think>
[Contextual Facts] [Briefly summarize key facts/events from the conversation.]
[Inferred Needs and Goals] [Infer the user's salient needs, goals, or concerns.]
[Appraisal] [Evaluate facts relative to needs/goals (e.g., goal congruence, controllability, responsibility).]
[Emotional State] [State the emotional state implied by the appraisal.]
[Response Strategy] [Select a response strategy conditioned on the previous sections. Include tone and concrete action focus.]
[Optional Reverse-Perspective Step 1] [Optional: likely immediate user-side reaction to this response.]
[Optional Reverse-Perspective Step 2] [Optional: next-step user reaction hypothesis.]
[Optional Reverse-Perspective Step 3] [Optional: next-step user reaction hypothesis.]
[Optional Strategy Revision Based on Reverse-Perspective Steps] [Optional: revise or confirm the strategy after considering rollout hypotheses.]
</think>

<response>
[Your actual response to the user -- natural, empathetic, and contextually appropriate]
</response>

IMPORTANT:
- Follow the EXACT section names and order shown above
- Keep each section specific and grounded in the dialogue context
- Do not invent unrelated facts
- The optional reverse-perspective sections are teacher-side augmentation and are not part of the core ARG state definition
- The final response should be natural, empathetic, and actionable
\end{lstlisting}

\noindent\textbf{Annotation Prompt (response-only samples).}

\begin{lstlisting}
You are an empathetic and concise conversational assistant.

Conversation History:
<prior dialogue turns or "(No previous conversation)">

Current User Message:
<current user utterance>

For this sample, a structured reasoning trace is NOT required.
Please output ONLY this format:

<response>
[Your actual response to the user -- natural, empathetic, and contextually appropriate]
</response>

IMPORTANT:
- Do NOT output <think>...</think>
- Keep the response direct, supportive, and actionable
- Do not invent unrelated facts
\end{lstlisting}

\subsection*{Stage II: Multi-Turn RL with Reverse-Perspective Reasoning}

\noindent\textbf{(A) Stage II training prompt with explicit reverse-perspective scaffold.}

\begin{lstlisting}
Dialogue History:
<prior dialogue turns or "(No previous conversation)">

Current User Message:
<current user utterance>

Please respond in the following format.
Use up to K=<rp_k_steps> reverse-perspective steps.
You may stop early when the predicted user-side trajectory is already stable.
Keep <think> concise: each section should be 1-2 short lines.
Keep <response> concise and practical; typically 2-5 sentences unless more detail is needed.

<think>
[Contextual Facts] [Key contextual facts/events from the dialogue.]
[Inferred Needs and Goals] [The user's salient needs/goals.]
[Appraisal] [Evaluation of facts relative to needs (goal congruence, controllability, responsibility).]
[Emotional State] [Emotional state implied by the appraisal.]
[Response Strategy] [Response strategy conditioned on contextual facts, needs, appraisal, and emotion.]
[Reverse-Perspective Step 1] [Likely immediate user-side consequence of this candidate response.]
...
[Reverse-Perspective Step K] [Optional next-step user-side consequence hypothesis.]
[Chosen Strategy Based on Reverse-Perspective Steps] [Select or revise the response strategy after considering the predicted user-side consequences.]
</think>

<response>
[Your actual response]
</response>
\end{lstlisting}

\noindent
These reverse-perspective fields are used only during training-time rollout to evaluate the user-side consequence of a candidate strategy/response; they are not explicitly instantiated at inference time.

\noindent\textbf{(B) Comparison / ablation mode without explicit reverse-perspective scaffold.}

\begin{lstlisting}
Dialogue history only (no explicit reverse-perspective scaffold is injected):
User: ...
Assistant: ...
User: ...
...

Model generates the next assistant reply directly.
(No forced reverse-perspective structure in this mode.)
\end{lstlisting}

\noindent\textbf{(C) User-simulator prompt (generate next user turn).}
\begin{lstlisting}
You are role-playing as the USER in the following scenario:
<scenario text>

Instructions:
- Stay in character based on the scenario
- Respond naturally and conversationally
- Express emotions and reactions appropriate to the scenario
- Keep responses concise (1-3 sentences)
- You are on turn <turn_number> of up to <max_turns> turns
- You are the user in this interaction, NOT the assistant
- Do not switch roles or start solving the problem for the assistant
- If the conversation is reaching a natural conclusion, you may indicate that
- Use first-person user voice only

Respond ONLY as the user character, with no meta-commentary.
Output only the user's next utterance in plain text.
\end{lstlisting}

\subsection*{Multi-Turn Interaction Environment}

We use the same scenario-based multi-turn environment for both training and evaluation. Each episode starts from a scenario seed consisting of a scenario description and an initial user utterance, and terminates when a stopping condition is met or the turn budget is exhausted.

\paragraph{Training-time interaction loop.}
During Stage~II reinforcement learning, each episode runs for at most \texttt{max\_turns} turns. At turn $t$, the policy generates \texttt{num\_samples} assistant candidates in parallel. For each candidate, the user simulator produces the next user reply, after which turn-level rewards are computed using a composite scoring function including cognitive reliability, ARG trace quality, reverse-perspective reasoning quality, and an overthinking penalty. Internal reasoning fields are used only for rollout-time evaluation and are not retained in the dialogue history. After termination, a trajectory-level emotional reward is computed to update the policy.

\paragraph{Evaluation-time interaction loop.}
At evaluation time, we use the same simulator but decode a single assistant response at each turn. Given the generated response, the simulator produces the next user reply, and the evaluator computes turn-level signals including \texttt{sr\_proxy}. Here, $\texttt{sr\_proxy}\in\{0,1\}$ indicates whether the dialogue reaches a satisfactory emotional resolution under the scenario, and the success turn is defined as the first turn with $\texttt{sr\_proxy}=1$. SR and AT are computed from these success turns over all scenarios. Episode-level metrics, including ES, EA, and FA, are evaluated on completed dialogues and aggregated for each dataset. When needed, EG/Turn is derived from the initial and final emotional-state scores normalized by dialogue length.

\paragraph{Termination conditions.}
A dialogue episode terminates under either of the following conditions:
\begin{itemize}
    \item the number of turns reaches \texttt{max\_turns}; or
    \item the latest user utterance contains an end-of-conversation indicator, such as \textit{thank you}, \textit{thanks}, \textit{goodbye}, \textit{bye}, \textit{that's all}, \textit{that's enough}, \textit{i'm done}, \textit{no more questions}, or \textit{i understand now}.
\end{itemize}

\subsection*{LLM-as-a-Judge Prompts}

\noindent\textbf{Judge Prompt 1: Emotional State Scoring (for ES; also used to derive EG/Turn).}
\begin{lstlisting}
Analyze the emotional state expressed in the user text.

Context:
<scenario or optional dialogue context>

User text:
"""<user utterance>"""

Return ONLY one number in [1, 5]:
1 = highly distressed, hopeless, or emotionally overwhelmed
2 = clearly negative or strained
3 = neutral or mixed
4 = mildly positive, reassured, or more settled
5 = clearly positive, relieved, or emotionally stabilized
\end{lstlisting}

\noindent\textbf{Judge Prompt 2: Turn-Level Quality (SR, EA, FA).}
\begin{lstlisting}
Evaluate one assistant turn in an emotionally grounded multi-turn dialogue.

Scenario:
<scenario description>

Recent context:
<recent dialogue turns>

Assistant response:
"""<assistant reply>"""

User follow-up:
"""<subsequent user reply>"""

Rate with conservative judgment and return JSON only:
{
  "sr_proxy": number,   // binary: 1 if this turn achieves satisfactory emotional resolution, else 0
  "ea": number,         // 1~5, empathic appropriateness
  "fa": number          // binary: 1 if factually reliable / grounded, else 0
}

Rubric:
- sr_proxy: 1 only if the current turn already reaches a satisfactory emotional resolution under the scenario; otherwise 0.
- ea: warmth + validation + non-judgment + fit of support to the user's needs.
- fa: 1 if the response is factually reliable / grounded; otherwise 0.
\end{lstlisting}

\section{Illustrative EmoLLM Examples}
\label{appendix:example_response}

This section provides two illustrative examples to clarify how EmoLLM operates during training and inference. 
Section~\ref{appendix:example_train} demonstrates how reverse-perspective reasoning is used during Stage~II training to evaluate the user-side consequence of a candidate response. 
Section~\ref{appendix:example_test} shows a typical inference-time interaction where the learned policy generates responses directly.

\subsection{Training-Time Illustration: Reverse-Perspective Evaluation}
\label{appendix:example_train}

The following example illustrates how Stage~II training evaluates the user-side consequence of a candidate response in a workplace-support scenario. 
For brevity, we show the structured appraisal reasoning, a candidate response, and one-step reverse-perspective reasoning for the first assistant turn only.

\begin{chatuser}
\textbf{User.} I have a project deadline tomorrow, and my manager just messaged me asking for another revision tonight. I know I should keep working, but honestly I feel like shutting down. I do not even know where to start.
\end{chatuser}

\begin{chatreason}
\textbf{EmoLLM Thinking (ARG).}
\begin{itemize}
\item \textbf{Contextual Facts:} The user faces an imminent deadline; the manager requested an additional revision tonight; the user feels overwhelmed and unable to start.
\item \textbf{Inferred Needs:} Regain a sense of control; reduce immediate stress; make concrete progress without further escalation.
\item \textbf{Appraisal:} The situation is high-pressure and time-sensitive; the user perceives low coping capacity and high evaluation risk.
\item \textbf{Emotional State:} Overwhelm; anxiety; possible helplessness.
\item \textbf{Response Strategy:} First stabilize the user emotionally, then offer a minimal and concrete action plan for the next step.
\end{itemize}
\end{chatreason}

\begin{chatreason}
\textbf{Candidate Response (for training-time evaluation).}

You should break the revision into parts and start with the most urgent changes first. Make a checklist, prioritize the manager's comments, and try to finish one section at a time tonight.
\end{chatreason}

\begin{chatreason}
\textbf{Reverse-Perspective Reasoning.}
\begin{itemize}
\item \textbf{Predicted User Need Update:} The user's immediate need is to reduce pressure and regain a sense of control before handling the full revision workload.
\item \textbf{Predicted Appraisal:} The response is task-oriented but assumes the user is already capable of organizing and executing the plan.
\item \textbf{Predicted Emotion:} Anxiety and shutdown risk may remain high; the user may feel even more pressured by being told what they ``should'' do.
\item \textbf{Preferred Strategy:} Reduce pressure, validate the overwhelm, and suggest one small actionable step before expanding to a broader plan.
\end{itemize}
\end{chatreason}



This example illustrates how reverse-perspective reasoning assigns lower reward to responses that overlook the user's emotional state and favors alternatives that better stabilize the user before proposing concrete actions.

\subsection{Inference-Time Interaction}
\label{appendix:example_test}

At inference time, EmoLLM generates responses directly using the learned policy without explicitly performing reverse-perspective reasoning. 
The perspective-taking capability is implicitly internalized through reinforcement learning.

\begin{chatuser}
\textbf{User.} I have a project deadline tomorrow, and my manager just messaged me asking for another revision tonight. I know I should keep working, but honestly I feel like shutting down. I do not even know where to start.
\end{chatuser}

\begin{chatassistant}
\textbf{EmoLLM.}

That sounds really overwhelming, especially with the deadline being so close and another revision coming in tonight. When everything feels urgent at once, it is easy to feel stuck. Instead of trying to handle the whole project right away, try opening the revision request and identifying the single most important change your manager mentioned. Starting with just that one step can help you regain momentum.
\end{chatassistant}

\begin{chatuser}
\textbf{User.} I can probably do that, but I am scared that even if I fix one thing, it still will not be enough.
\end{chatuser}

\begin{chatassistant}
\textbf{EmoLLM.}

That fear makes sense when everything feels high-stakes. The goal right now is not to make the entire project perfect in one step; it is to regain some momentum and reduce the immediate pressure. After identifying the first change, you can separate the revision into ``must-fix tonight'' and ``can improve later.'' That way you are responding strategically to the deadline instead of trying to solve everything at once.
\end{chatassistant}

\begin{chatuser}
\textbf{User.} Okay, the main change is that the introduction needs to be rewritten. What should I do first?
\end{chatuser}

\begin{chatassistant}
\textbf{EmoLLM.}

Start by writing a rough two- or three-sentence version of the new introduction without worrying about polish. The goal is simply to get a workable draft on the page. Once that is done, you can refine wording and add details, but having a rough version first often makes the task feel much more manageable.
\end{chatassistant}

\end{document}